\newif\ifdraft
 \draftfalse

\documentclass[11pt]{article}

\usepackage[letterpaper,margin=1in]{geometry}
\usepackage[parfill]{parskip}

\usepackage{authblk}

\usepackage[toc,page,header]{appendix}

\usepackage[utf8]{inputenc} 
\usepackage[T1]{fontenc}    
\usepackage[colorlinks,citecolor=blue,urlcolor=blue,linkcolor=blue,linktocpage=true]{hyperref}       
\usepackage{url}            
\usepackage{booktabs}       
\usepackage{amsfonts}       
\usepackage{nicefrac}       
\usepackage{microtype}      
\usepackage{xcolor}         
\label{key}
\usepackage{thmtools}

\usepackage{comment}

\usepackage{epsfig}
\usepackage{graphicx}
\usepackage{wrapfig}
\usepackage{subfigure}
\usepackage{multirow}
\usepackage{hyperref}
\usepackage{graphicx}
\usepackage{caption}
\usepackage{array}
\usepackage{bbm}
\usepackage{color}
\usepackage{enumerate}
\usepackage{enumitem}
\setlist{leftmargin=10mm}
\usepackage{mathtools}
\usepackage{amsmath,amssymb,amsthm,bm}
\usepackage{amsfonts,graphicx}
\usepackage{mathrsfs}
\usepackage{algorithmic,algorithm}

\usepackage[authoryear]{natbib}

\def\pr{\mathrm{Pr}}
\def\E{\mathbb{E}}

\def \E{\mathbb{E}}

\def\R{\mathbb{R}}

\def\cD{\mathcal{D}}

\def\cF{\mathcal{F}}
\def\cG{\mathcal{G}}
\def\cH{\mathcal{H}}
\def\cI{\mathcal{I}}
\def\cJ{\mathcal{J}}
\def\cO{\mathcal{O}}

\def\cM{\mathcal{M}}
\def\cN{\mathcal{N}}

\def\cR{\mathcal{R}}

\def\cZ{\mathcal{Z}}

\newcommand{\diff}{\,\mathrm{d}}
\newcommand{\e}{\mathrm{e}}
\newcommand{\eps}{\varepsilon}
\usepackage{cleveref}
\newtheorem{theorem}{Theorem}
\newtheorem{lemma}[theorem]{Lemma}

\newtheorem{definition}[theorem]{Definition}
\newtheorem{example}[theorem]{Example}
\newtheorem*{remark}{Remark}
\newenvironment{claim}[1]{\par\noindent\underline{Claim:}\space#1}{}
\newtheorem{remark-star}{Remark}
\newtheorem{remark-star-1}{Remark}

\newtheorem{corollary}[theorem]{Corollary}
\newtheorem{proposition}[theorem]{Proposition}
\newtheorem*{proof-sketch}{Proof Sketch}
\usepackage{booktabs}

\author[1]{Yuqing Zhu}
\author[2]{Jinshuo Dong}
\author[1]{Yu-Xiang Wang}
\affil[1]{Computer Science Department, UC Santa Barbara}
\affil[2]{IDEAL Institute, Northwestern University}

\title{Optimal Accounting of Differential Privacy via Characteristic Function}
\begin{document}



\maketitle

\begin{abstract}
	
	Characterizing the privacy degradation over compositions, i.e., privacy accounting, is a fundamental topic in differential privacy (DP) with many applications to differentially private machine learning and federated learning.  We propose a unification of recent advances (Renyi DP, privacy profiles, $f$-DP and the PLD formalism) via the \emph{characteristic function} ($\phi$-function) of a certain \emph{dominating} privacy loss random variable. We show that our approach allows \emph{natural} adaptive composition like Renyi DP,  provides \emph{exactly tight} privacy accounting like PLD, and can be (often \emph{losslessly}) converted to privacy profile and $f$-DP, thus providing $(\epsilon,\delta)$-DP guarantees and interpretable tradeoff functions.  Algorithmically, we propose an \emph{analytical Fourier accountant}\footnote{Code is available at https://github.com/yuxiangw/autodp} that represents the \emph{complex} logarithm of $\phi$-functions symbolically and uses Gaussian quadrature for numerical computation. On several popular DP mechanisms and their subsampled counterparts, we demonstrate the flexibility and tightness of our approach in theory and experiments.
	
\end{abstract}

\section{Introduction}

Differential privacy (DP) \citep{dwork2006calibrating} is one of the most promising approaches towards addressing the privacy challenges in the era of artificial intelligence and big data. Recently, DP is going through an exciting transformation from a theoretical construct into a practical technology \citep[see, e.g., ][]{apple2017,erlingsson2014rappor,census2017}, which demands constant-tight privacy accounting tools that use the privacy budget with optimal efficiency.

Much of the progress in the recent theory and practice of DP has been driven by Renyi Differential Privacy (RDP) \citep{mironov2017renyi}, e.g., it is the major technical component behind  the \emph{first practical method} for \emph{deep learning with differential privacy} \citep{abadi2016deep}.  More broadly, RDP is among several recent work in differential privacy that conducts fine-grained mechanism specific analysis  
\citep{bun2016concentrated,abadi2016deep,mironov2017renyi,balle2018improving,wang2019subsampled,dong2019gaussian, sommer2019privacy,koskela2020computing}. 
At the heart of these breakthroughs is the idea of using a \emph{function} to describe the privacy guarantee of a randomized procedure, thus produces significantly more favorable privacy-utility tradeoff and tighter bounds under composition. (See Table~\ref{tab:functional_views} for a summary their pros and cons). 

\begin{table*}[h]
	\centering
	\resizebox{\textwidth}{!}{ 
		\begin{tabular}{|c|c|c|p{4.8cm}|}
			\hline
			&Functional view & Pros & Cons \\
			\hline
			Renyi DP \citep{mironov2017renyi} & $D_{\alpha}(P\| Q)\leq \epsilon(\alpha), \forall \alpha\geq 1$& Natural composition & lossy conversion to $(\epsilon,\delta)$-DP.\\
			Privacy profile
			\citep{balle2018improving} & $\E_q[(\frac{p}{q} - e^\epsilon)_+] \leq \delta(e^\epsilon), \forall \epsilon\geq 0$  & Interpretable. & messy composition.\\
			$f$-DP\citep{dong2019gaussian} &  Trade-off function $f$ & Interpretable, CLT & messy composition.\\
			PLD \citep{sommer2019privacy,koskela2020computing} & Probability density of $\log (p/q)$ & Natural composition via FFT &  Limited applicability.\\
			\hline
		\end{tabular}
	}
	\caption{Modern functional views of DP guarantees and their pros and cons.}\label{tab:functional_views}
\end{table*}

Note that no single approach dominates others in all dimensions.  
Renyi DP could be undefined for certain privacy loss distributions, and cannot be used to provide the optimal $(\epsilon, \delta)$-DP computation in general (discussed in Section~\ref{sec: lossy_rdp}).
Privacy profiles and $f$-DP are unwieldy under composition; and the method of \citep{koskela2020computing} is limited to mechanisms with univariate output where $\log(p/q)$ admits a density; or those with discrete outputs. Usually, for a new mechanism, we would be lucky to have any one of these functional descriptions.  The need to derive these manually for each new mechanism is clearly limiting the creativity of researchers and practitioners in DP. 

In addition, there are some unresolved foundational issues related to the PLD formalism. As is repeatedly articulated by the authors, the PLD formalism is defined for \emph{each pair} of neighboring datasets separately, thus, strictly speaking, does not imply DP unless we can certify that the pair of neighboring datasets is the worst-case. This is challenging because such a pair of datasets might not exist and it is unclear how we can define a partial ordering of two privacy loss distributions. 

In this paper, we provide a unified treatment to these functional representations and resolve the aforementioned subtle issues related to the PLD formalism. Our contributions are summarized below.
\begin{enumerate}[leftmargin=*]
	\item We formalize and generalize the notion of \emph{``worst-case'' pair distributions} discussed in \citep{sommer2019privacy} to a \emph{``dominating pair''} and prove several basic properties of the \emph{dominating pairs} including finding such pairs from any privacy-profiles,  \emph{adaptive} composition and \emph{amplification by sampling}. These results substantially broaden the applicability of PLD formalism \citep{sommer2019privacy} in deriving  \emph{worst-case} DP guarantees.
	\item We propose a lossless representation of the privacy loss RV by its characteristic function ($\phi$-function) and derive optimal conversion formula to (and from) privacy-profile, tradeoff-function ($f$-DP) and the distribution function of the privacy loss RV. Many of these conversion rules correspond naturally to the classical Fourier / Laplace transforms (and their inverses) from the signal processing literature. 
\item We design an Analytical Fourier Accountant (AFA, extending the Fourier accountant of \citep{koskela2020computing,koskela2020tight}) which represents the complex logarithm of the  $\phi$ function \emph{symbolically}. AFA can be viewed as an extension of the (analytical) moments-accountant \citep{abadi2016deep,wang2019subsampled} to complex $\alpha$, thus allowing straightforward composition.  Computing $\delta$ as a function of $\epsilon$ for $(\epsilon,\delta)$-DP boils down to a numerical integral which we use a Gaussian quadrature-based method to solve efficiently and accurately. 
\item Experimentally, we demonstrate that our approach provides substantially tighter privacy guarantees over compositions than RDP on both basic mechanisms and their subsampled counterparts. Our results essentially match the results from \citep{dong2019gaussian} and \citep{koskela2020tight} but neither rely on central-limit-theorem type asymptotic approximation nor require choosing appropriate discretization a priori as in the FFT-based Fourier Accountant. 
\end{enumerate}

\textbf{Related work:}  The paper builds upon the existing work on RDP-based privacy accounting \citep{abadi2016deep,mironov2017renyi,wang2019subsampled} as well as $f$-DP \citep{dong2019gaussian}. Our main theoretical contribution is to substantially broaden the applicability of the PLD formalism \citep{sommer2019privacy} by proposing the notion of \emph{dominating pairs} and providing \emph{general recipes} for constructing these dominating pairs.
The closest to algorithmic contribution is the work of  \citet{koskela2020computing,koskela2020tight}, who propose Fourier accountant and an FFT-based approximation scheme, the characteristic function view can be seen  as an \emph{analytical} version of their Fourier accountant (hence the name AFA).  AFA is more generally applicable, and allows more flexible use of existing methods for numerical integral. The recent work of \citet{gopi2021numerical} improves the FFT accountant substantially. It is complementary to us in that it does not address the foundational issues of the PLD formalism, nor do they propose an analytical representation that allows a more modular design of the privacy accountant. Notably, we can use any blackbox numerical integration tool, e.g., Gaussian quadrature, and set the desired error bound on-the-fly, while an FFT-accountant requires setting the parameters at initialization. Finally, \citet{canonne2020discrete} considered $\phi$ function and its numerical / computational properties but the discussion is restricted to the discrete Gaussian mechanism.

Privacy accounting is closely related to the classical advanced composition of $(\epsilon,\delta)$-DP \citep{dwork2010boosting}; \citet{kairouz2015composition} provides the \emph{optimal} $k$-fold composition of an $(\epsilon,\delta)$-DP mechanism and \citet{murtagh2016complexity} shows that computing the tightest possible bound for the composition of $k$ heterogeneous mechanisms is $\#P$-hard. The recent line of work (that we are building upon) challenges the basic primitive of composing $(\epsilon_i, \delta_i)$-DP by composing certain functional descriptions of the mechanisms themselves, which sometimes avoids the computational hardness (but not always) and results in even stronger composition than the best $(\epsilon,\delta)$-DP type composition would allow \citep{bun2016concentrated}. 




\vspace{-2mm}
\section{Notations and preliminary}
\vspace{-2mm}
In this section, we review the standard definition of differential privacy, its RDP relaxation, introduce the characteristic function and draw connections with RDP. 

\noindent \textbf{Symbols and notations}. Throughout the paper, we will use standard notations for 
probability unless otherwise stated, e.g., $\pr[\cdot]$ for probability, $p[\cdot]$ for density, $\E[\cdot]$ for expectation, $F[\cdot]$ for CDF. $\epsilon, \delta$ are reserved for privacy budget/loss parameters as in $(\epsilon, \delta)$-DP, except in the cases when we write $\epsilon(\cdot)$ or $\delta(\cdot)$, where they become functions of certain arguments. We will use $D, D' \in \cD^* := \cup_{n \in \mathbb{N}}\cZ^n$ to denote two datasets with an unspecified size. $D, D'$ are neighboring (denoted by $D\simeq  D'$) if we can construct $D'$ by adding or removing one data point from $Z$. $\cM: \cD^* \to P_{\cO}$ is a randomized mechanism which returns an output $o \in \cO$ by sampling from distribution $\cM(D)$. Sometimes for convenience and clarity we define $P, Q$ and $p, q$ to be the distribution and density functions of $\cM(D)$ and $\cM(D')$ respectively.

\noindent \textbf{Differential privacy and its equivalent definitions.} With these notations clarified, we can now formally define differential privacy.
\begin{definition}[Differential Privacy]\label{def:DP}
	A randomized algorithm  $\cM$ is $(\epsilon, \delta)$-DP if for every pair of neighboring datasets $D, D'$, and every possible output set $S \subseteq \cO$ the following inequality holds:
	$$
	\pr[\cM(D)\in S] \leq e^{\epsilon}\pr[\cM(D'\in S)] + \delta.
	$$
\end{definition}
We can alternatively interpret DP from the views of  a divergence metric of two probability distributions, a hypothesis testing view of a binary-classifier, as well as the distribution of the log-odds ratio.  Let us first define these quantities formally.
\begin{definition}[Hockey-stick divergence]\label{def: hockey}
	For $\alpha >0$, the Hockey-stick divergence is defined as 
	$H_\alpha(P\|Q)  :=  \E_{o\sim Q}[ (\frac{\diff P}{\diff Q}(o) - \alpha)_+ ],$
	where $(x)_+  :=  x  \mathbf{1}(x \geq 0)$ and $\frac{\diff P}{\diff Q}$ is the Radon-Nikodym-derivative (or simply the density ratio when density exists for $P$ and $Q$). 
\end{definition}
\begin{definition}[Trade-off function] 
Let $\phi$ be a classifier to distinguish two distributions $P$ and $Q$ using a sample. $\alpha_\phi$ be its Type I error (false positive rate) and $\beta_{\phi}$ be its Type II error (false negative rate). The tradeoff function $T_{P,Q}(\alpha): [0,1]\rightarrow [0,1]$ is defined to be
	$
	T_{P,Q}(\alpha) := \inf_{\phi}\{ \beta_\phi \;|\;  \alpha_{\phi} \leq \alpha \}.
	$
\end{definition}
\begin{definition}[Privacy loss R.V.]\label{def:privacyloss_rv} The \textit{privacy loss random variable} of for a pair of neighboring dataset $D,D'$ under mechanism $\cM$ is defined as
	$
	L_{P,Q}:= \log\frac{\cM(D)(o)}{ \cM(D')(o)} \text{  where  }  o \sim \cM(D);
	$
	similarly, we have  $L_{Q,P}:= \log\frac{\cM(D')(o)}{ \cM(D)(o)} \text{  where  }  o \sim \cM(D').$
\end{definition}

These quantities can be used to equivalently define differential privacy \citep{wasserman2010statistical,barthe2013beyond,kairouz2015composition,balle2018improving,balle2018couplings,dong2019gaussian}.
\begin{lemma}\label{lem:equivalent_defs}
	The following statements about a randomized algorithm $\cM$ are equivalent to Definition~\ref{def:DP}
	\begin{enumerate}[leftmargin=*]
		\item $
		\sup_{D\simeq D'} H_{e^\epsilon}(\cM(D)\|\cM(D'))  \leq \delta. 
		$
				\item $\sup_{D\simeq D'} T_{\cM(D),\cM(D')}(\alpha) \geq \max\{0, 1-\delta - e^\epsilon \alpha, e^{-\epsilon}(1-\delta - \alpha).$
		\item $
		\Pr_{o\sim\cM(D)}[L_{P, Q}>\epsilon]  - e^{\epsilon}\Pr_{o\sim \cM(D')}[L_{Q, P}< -\epsilon] \leq \delta
		$ for all neighboring $D,D'$.
	\end{enumerate}
\end{lemma}
We highlight that in all these definitions, it is required for the bound to cover all pairs of neighboring datasets $D,D'$.

\noindent \textbf{Mechanism-specific analysis / Functional representation of DP guarantee.}
Each of these equivalent interpretations could be used to provide more-fine-grained description of a differential privacy mechanism $\cM$. For instance, the privacy profile $\delta_{\cM}(\epsilon)$ upper bounds the HS-divergence for all $\epsilon$ and the $f$-DP lowerbounds the tradeoff function for all Type I error $\alpha$ (see Table~\ref{tab:functional_views}).
In addition, \citet{sommer2019privacy} proposes the PLD formalism, which represents the privacy loss RV by its density function. The PLD formalism can be viewed as another functional representation, but it is qualitatively different from privacy profile and $f$-DP. We will expand further on PLD in Section~\ref{sec: lossy_rdp}.

\noindent \textbf{Renyi Differential Privacy and Moments Accountant.} 
Renyi differenital privacy (RDP) \citep{mironov2017renyi} is another generalization of pure-DP via Renyi divergence (denoted by $\cD_{\alpha}(P ||Q)$).
\begin{definition}[Renyi Differential Privacy \citep{mironov2017renyi}]
	We say a randomized algorithm $\cM$ is $(\alpha, \epsilon(\alpha))$-RDP with order $\alpha \geq 1$ if for neighboring datasets $D, D'$,
	$$
\cD_{\alpha}(\cM(D)||   \cM(D')) \leq \epsilon(\alpha).
$$
\end{definition}
$(\epsilon,\alpha)$-RDP implies $(\epsilon(\alpha)+\frac{\log(1/\delta)}{\alpha-1},\delta)$-DP, thus by viewing RDP as a function $\epsilon_{\cM}(\cdot)$, we can find the best $\epsilon$ parameter by optimizing over $\alpha$. Tighter conversion formula had been proposed recently \citep{balle2020hypothesis,asoodeh2021three}, which we discuss in Appendix.

%
%
%
The main advantage of RDP is that it composes \emph{naturally} over multiple adaptively chosen mechanisms via a straightforward rule $\epsilon_{\cM_1 \times \cM_2} \leq \epsilon_{\cM_1} + \epsilon_{\cM_2}$. It recovers the advanced composition when converting to $(\epsilon,\delta)$-DP and yields substantial additional savings. These properties, together with the privacy-amplification by sampling, makes RDP the natural choice for privacy accounting in various algorithms of differentially private deep learning.   The related algorithm that keeps track of the moment generating function of $L_{P,Q}(o)$ is called ``moments accountant''  \citep{abadi2016deep,wang2019subsampled}. 

\section{Motivation of our research}~\label{sec: lossy_rdp}

In this section, we discuss a number of limitations of Renyi DP and PLD formalism that, in part, motivated our research.
\begin{figure*}[t]
	\centering	
	\subfigure[RDP of RR and GM]{
		\includegraphics[width=0.31\textwidth]{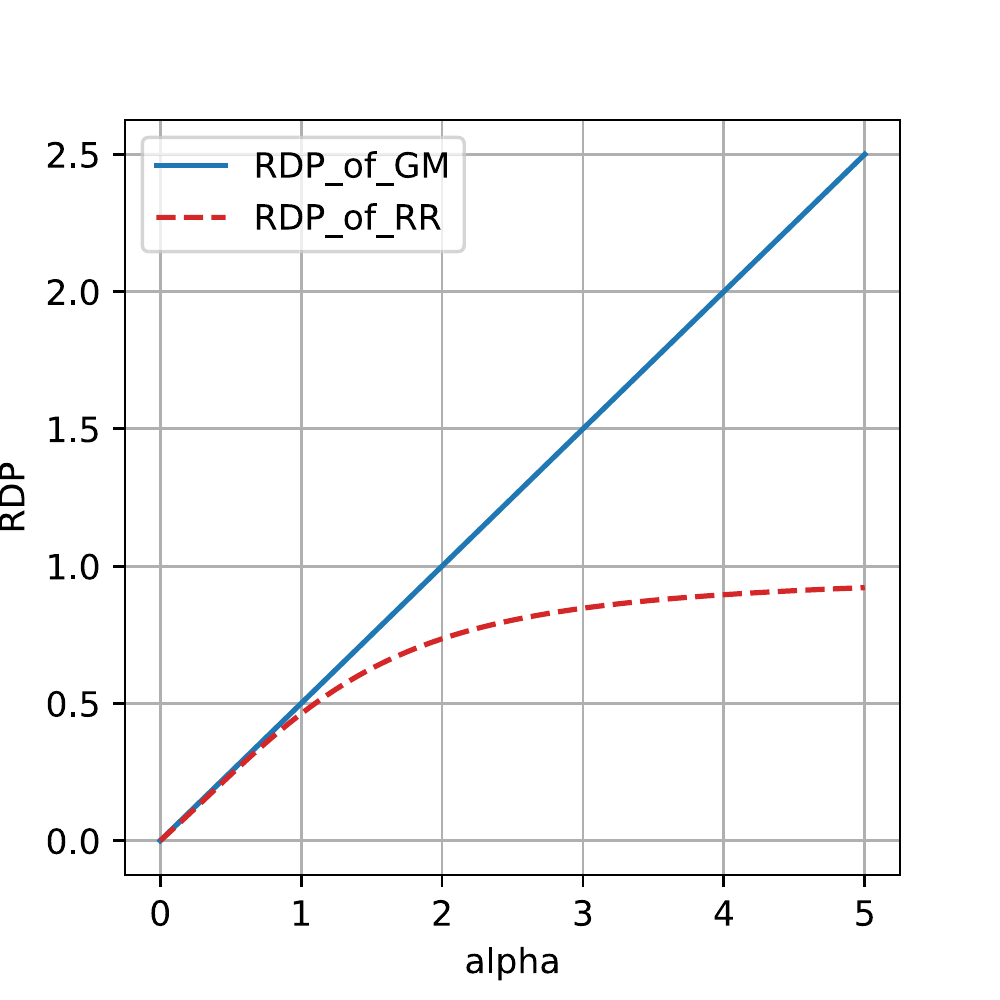}}
	\subfigure[$f$-DP of RR and GM]{
		\includegraphics[width=0.31\textwidth]{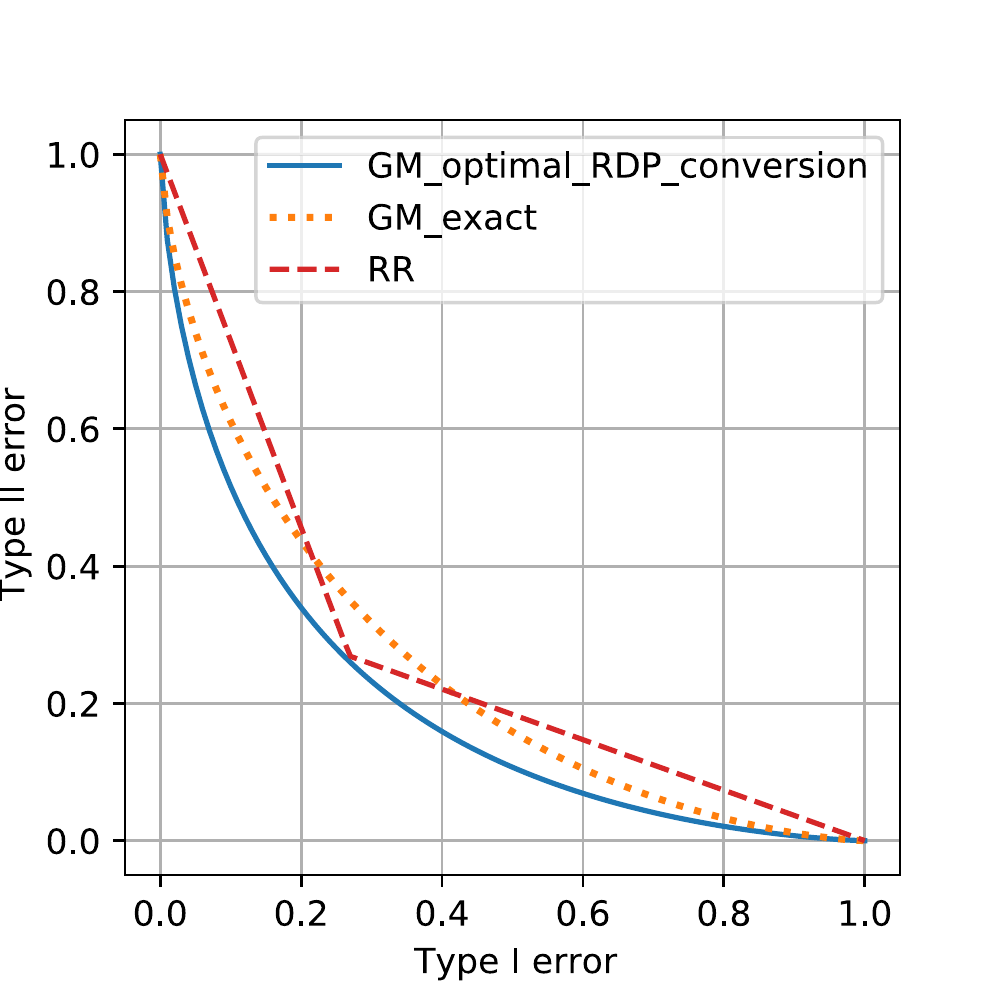}}
	\subfigure[$(\epsilon,\delta)$-DP of RR and GM]{
		\includegraphics[width=0.31\textwidth]{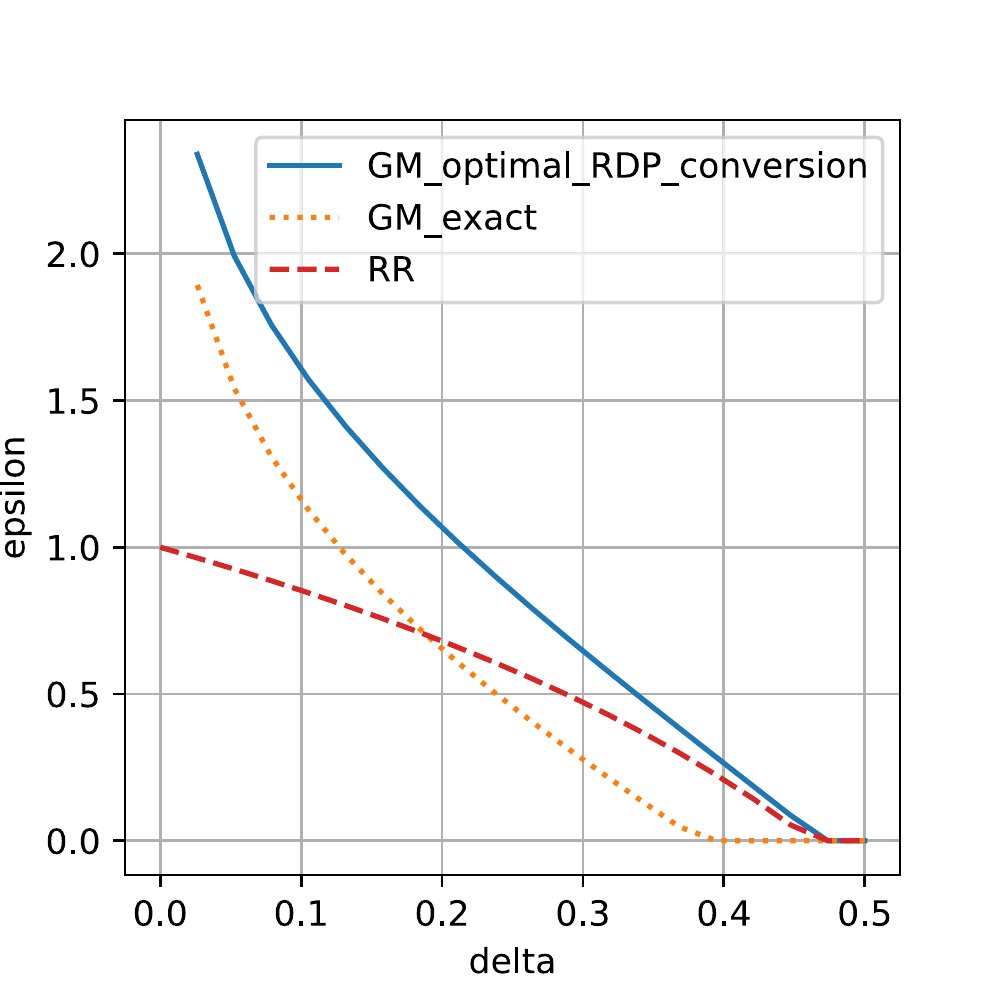}}
	\vspace{-3mm}
	\caption{ The figure illustrates the RDP and $f$-DP of a Gaussian mechanism with (normalized) $\sigma = 1$, and a randomized response mechanism with $p = \frac{e}{1+e}$.  Pane (a) shows the RDP function of RR and GM, clearly, RR also satisfies the same RDP of the Gaussian mechanism for all $\alpha$.  Pane (b) in the middle compares the $f$-DP of the two mechanisms, as well as the $f$-DP implied by the optimal conversion from RDP. Pane (c) shows the privacy profile of the two mechanisms, together with Pane (a), it demonstrates that the optimal $f$-DP and $(\epsilon,\delta)$-DP of GM cannot be achieved by a conversion from RDP.  
	} 
		\vspace{-3mm}
	\label{fig: lossy_rdp}
\end{figure*}
%
%

\noindent\textbf{The limits of RDP.}
Let us first ask ``is the RDP function a \textit{lossless} description?'' In particular, does it capture all information in the privacy-profile? Because if it is the case, then we could use RDP for composition, and then find the exact optimal $(\epsilon,\delta)$-DP by converting from RDP.
 
The answer is unfortunately ``no''.  The reasons are twofolds.  First, there are mechanisms with non-trivial $(\epsilon,\delta)$-DP where RDP parameters partially or entirely do not exist.  We give two concrete examples 
 in Appendix~\ref{sec:limits_rdp}. 
 
The second, and a more troubling issue is that even in the cases when RDP parameters exist everywhere and hence \emph{appears to be} characterizing, it does not lead to a tight conversion to $(\epsilon, \delta)$-DP.  Gaussian mechanism is such a candidate  where its PLD is completely captured by its Renyi divergences.
However, in Figure~\ref{fig: lossy_rdp} we demonstrate 
that we cannot, in general, convert the RDP of Gaussian mechanism into an $(\epsilon,\delta)$-DP that matches the optimal accounting one can achieve through either the privacy profile or $f$-DP directly. Specifically, by an example due to \citep[Proposition B.7]{dong2019gaussian},  we know that a randomized response mechanism (RR) satisfies $1$-zCDP, thus the same RDP as that of a Gaussian mechanism (GM) with $\sigma = 1$. If the RDP conversion is tight, then it will have to apply to RR too, but that will lead to a contradiction with the tradeoff function of RR.  More explicitly, when we further convert the $f$-DP in Figure~\ref{fig: lossy_rdp} to $(\epsilon,\delta)$-DP, this example shows that while both RR and GM satisfy an RDP with $\epsilon(\alpha) = \frac{\alpha}{2}$, GM obeys
 $(0.277,0.3)$-DP but RR does not satisfy $(\epsilon,0.3)$-DP with $\epsilon < 0.471$. 

This example certifies that the conversion rule we used (based on an extension of \citep{balle2020hypothesis}) \emph{cannot be improved} and that \emph{RDP is a lossy representation} even for the Gaussian mechanism.

\noindent\textbf{Trouble with Worst-Cases in the PLD formalism.}
%
Recent developments in the PLD formalism show great promises in computing tight ($\epsilon,\delta$)-DP with stable numerical algorithms and provable error bounds \citep{koskela2020computing,koskela2020tight}. However, as we discussed earlier, PLD is specified for each pair of input datasets separately. To use PLD, the original authors (quoting verbatim) \emph{``require the privacy analyst interested in applying our results (PLD formalism) to provide worst-case distributions.''} \citep[Section 2]{sommer2019privacy}.  In a subsequent work \citep{meiser2018tight}, a subset of the authors further derive the worst-case pair of distributions for basic mechanisms such as Gaussian mechanism and Laplace mechanism \citep{meiser2018tight}. 

While these are valid arguments, the line of  work on PLD formalism does not formally define the worst-case pair of distributions, nor do they provide general recipes for ``privacy analysts'' to determine which pair of inputs is the worst-case. The issue is more prominent when we consider mechanism-specific analysis, because the pairs of datasets that attain the argmax might be different in different regions of the privacy profile (see an example in Appendix~\ref{sec:limits_rdp}).


Moreover, in most typical use cases of the privacy accounting tools, the mechanism under consideration is constructed through the composition of a sequence of simpler mechanisms. Even if for each mechanism, we know the worst-case pair distributions, the composition of the individual PLDs may not correspond to the worst-case PLD of the composed mechanism \footnote{This is an issue we will address later, which shows that it is OK even if it does not.}. For this reason, it is unclear how to use PLD for deriving worst-case DP bound under composition except in highly specialized cases (e.g., Gaussian mechanisms and their compositions).


\noindent\textbf{Summary.} To reiterate, RDP is lossy when converting to $(\epsilon,\delta)$-DP and the PLD formalism cannot be used to handle the composition generically due to issues regarding worst-case distributions.  The remainder of the paper will be dedicated to addressing this dilemma.


%

\section{Main results}
In this section, we develop a comprehensive solution towards tighter and more flexible \emph{mechanism-specific} privacy accounting for $(\epsilon,\delta)$-DP with a data-structure that allows natural composition.

\subsection{Dominating pair of distributions, composition and subsampling}\label{sec:worst_case_calculus}
We first patch the PLD formalism by generalizing the idea of worst-case pair (which may not exist) to a \emph{dominating} pair of distributions and prove a number of useful properties.




\begin{definition}[Dominating pair of distributions]\label{de:worst-case-pair}
	We say that $(P,Q)$ is a \emph{dominating} pair of distributions for $\cM$ (under neighboring relation $\simeq$) if for all $\alpha\geq 0$\footnote{Note that $\alpha\geq 1$ corresponds to the typical range of $(\epsilon,\delta)$-DP, but the region for $\alpha<1$ is important for composition and lossless conversions to other representations.
	}
	\begin{equation}\label{eq:def-worst-case-pair}
	\sup_{D \simeq D'}  H_{\alpha}( \cM(D)\| \cM(D')  )  \leq  H_{\alpha}(P\|Q).
	\end{equation}
	When $P,Q$ is chosen such that \eqref{eq:def-worst-case-pair} takes ``$=$'' for all $\alpha$, we say that $(P,Q)$ is a \emph{tight} dominating pair of distributions or simply, \emph{tightly dominating}.
	If in addition, there exists a neighboring $(\tilde{D},\tilde{D}')$ such that $(\cM(\tilde{D}), \cM(\tilde{D}'))$ is \emph{tightly dominating}, and then we say  $(\tilde{D},\tilde{D}')$ is the worst-case pair of datasets for mechanism $\cM$.
\end{definition}
Unless otherwise specified, all subsequent results we present hold for any definitions of neighbors (including asymmetric ones such as \emph{add-only} and \emph{remove-only}, which will be useful later). 
%
%

A dominating pair of distributions always exists: one can trivially take $P$ and $Q$ that have disjoint supports. What is somewhat surprising is the following
\begin{proposition}\label{cor:tightly_dominating}
	Any mechanism has a \emph{tightly} dominating pair of distributions.
	\end{proposition}

For example, the domintating pair for discrete Gaussian mechanism (DGM)~\citep{canonne2020discrete} will be two discrete Gaussian, e.g., $P = \cN\mathbb{Z}(0, \sigma^2), Q=\cN\mathbb{Z}(\Delta, \sigma^2), \Delta\in \mathbb{Z}_{+}$ is the sensitivity of the integer-valued query. This follows because the probability mass of the discrete Gaussian is a log-concave sequence. The proof would look very similar to Proposition A.3 of \citet{dong2019gaussian}.
On the other hand, worst-case pair of datasets do not always exist, as is shown by \Cref{ex:worst_case_pair_inputs_do_not_exist}. 

\Cref{cor:tightly_dominating} is the direct consequence of the following result which fully characterizes what hockey-stick divergences and privacy profiles look like.
\begin{lemma}\label{prop:whatisHS}
	For a given $H:\R_{\geqslant0}\to\R$, there exists $P,Q$ such that $H(\alpha)=H_\alpha(P\|Q)$ if and only if $H\in\cH$ where
	{\small
	\begin{align*}
	\cH:=\left\{H:\R_{\geqslant0}\to\R\middle| \begin{aligned}
		&H\text{ is convex, decreasing, } \\
		&H(0)=1  \text{ and  } H(x)\geqslant(1-x)_+
	\end{aligned} \right\}.
	\end{align*}
}
	Moreover, one can explicitly construct such $P$ and $Q$: $P$ has CDF $1+H^*(x-1)$ in $[0,1)$ and $Q=\textup{Uniform}([0,1])$.
\end{lemma}

The proof, presented in Appendix~\ref{sec:proofs}, makes use of the Fenchel duality of the privacy profile with respect to a tradeoff function and a characterization of the tradeoff function due to \citet[Proposition 2.2]{dong2019gaussian}.


What makes the specific construction in \Cref{prop:whatisHS} (hence \Cref{cor:tightly_dominating}) appealing is that even if the output space is complex, the resulting dominating pair of distributions are of univariate random variables defined on $[0,1]$. This resolves a limitation of \citet{koskela2020computing} that requires the mechanism to have either univariate or discrete outputs.

So far, we have shown the existence of a tightly dominating pairs for all mechanisms (Proposition~\ref{cor:tightly_dominating}), and provided a recipe for constructing such a dominating pair for any valid upper bounds of the privacy profile (Lemma~\ref{prop:whatisHS} and Corollary~\ref{cor:dominating_pair} in Appendix~\ref{sec:proofs}). 
Next we will provide two general primitives on how to construct dominating pairs for more complex mechanisms created by composition and privacy amplification by sampling.
\begin{theorem}[Adaptive composition of dominating pairs]\label{thm:composition}
	If $(P,Q)$ dominates $\cM$ and $(P', Q')$ dominates $\cM'$\footnote{$\cM'$ can be adaptively chosen in that it could depend on the output of $\cM$, which requires $\sup_{o\in \textrm{Range}(\cM)}H_{\alpha}(\cM'(D,o)\| \cM'(D',o)) \leq H_{\alpha}( P'\|Q')$ for any value of $o$. 
	},  then 
	$
	(P\times P', Q\times Q')
	$
	dominates the composed mechanism $(\cM, \cM')$.
\end{theorem}
By induction, this theorem implies that if we construct the PLD using a dominating pair of distributions for each individual mechanism, then the composed PLD can be used to obtain a \emph{valid} worst-case DP of the composed mechanism. 

Next we present how we can construct a dominating pair of distributions (and datasets) for mechanisms under ``privacy-amplification by sampling''. This is a powerful primitive that is used widely in differentially private ERM \citep{bassily2014private}, Bayesian learning \citep{wang2015privacy} and deep learning \citep{abadi2016deep}. We consider the following two schemes.
\begin{description}[leftmargin=*]
	\itemsep=0em
	\item[Poisson Sampling] Denoted by $S_{\textbf{Poisson}}^\gamma$. $S_{\textbf{Poisson}}^\gamma$ takes a dataset of arbitrary size and return a dataset by including \emph{each} data point 
	with probability $0\leq \gamma \leq 1$ i.i.d. at random.
	\item[Subset Sampling]  Denoted by $S_{\textbf{Subset}}^\gamma$. $S_{\textbf{Subset}}^\gamma$ takes a dataset with size $n$ or $n-1$ and return a subset of size $m < n$  uniformly at random. We define $\gamma := m / n$ as a short-hand. \footnote{Note that here $n,m$ are public and $\gamma:=m/n$ even if $(n-1)$ is the sample size.} 
\end{description}
Somewhat unconventionally, the following theorem not only considers add/remove neighboring relation but also treat them \emph{separately}, which turns out to be crucial in retaining a tight dominating pair with a closed-form expression\footnote{See the appendix for a construction of dominating pairs of subsampled mechanisms under ``Add/Remove'' or ``Replace'' neighbors and more detailed discussion on the advantage of treating ``Add'' and ``Remove'' separately.}. Our choice of choosing $\alpha\geq 0$ in Definition~\ref{de:worst-case-pair} ensures that for any mechanism $(P, Q)$ dominates for add neighbors \emph{iff} $(Q, P)$ dominates for removal neighbors. 
\begin{restatable}[]{theorem}{amp}
\label{thm:amplification_by_sampling}
Let $\cM$ be a randomized algorithm.
\begin{enumerate}
	\item[\textup{(1)}] 
	{
	If $(P,Q)$ dominates $\cM$ for add neighbors then  $(P, (1-\gamma) P+\gamma Q)$ dominates $\cM\circ S_{\textbf{\textup{Poisson}}}$ for add neighbors and $((1-\gamma)Q + \gamma P, Q )$ dominates $\cM\circ S_{\textbf{Poisson}}$ for removal neighbors.
	}
	\item[\textup{(2)}] 
	{If $(P,Q)$ dominates $\cM$ for replacing neighbors, then $(P, (1-\gamma) P +\gamma Q)$ dominates $\cM\circ S_{\textbf{\textup{Subset}}}$ for add neighbors and $((1-\gamma) P +\gamma Q,P)$ dominates $\cM\circ S_{\textbf{\textup{Subset}}}$ for removal neighbors.}
\end{enumerate}
\end{restatable}
We can obtain the results for the standard "add/remove" for a $k$-fold composition of subsampled mechanism by a pointwise maximum of the two:
\[\max\{H_{e^\epsilon}(P_1^k || Q_1^k), H_{e^\epsilon}(P_2^k ||  Q_2^k))\}\]
where $(P_1, Q_1)$ is the ``remove only'' version of dominating pair and $(P_2, Q_2)$ is the ``add only'' version of dominating pair. 
Existing literature that uses PLD for Poisson-sampled mechanisms while taking $(\gamma P + (1-\gamma)Q, Q)$ as an input are essentially providing privacy guarantees only for the ``remove only'' neighboring relationship.
To the best of our knowledge, this is the first time a dominating pair of distributions under privacy-amplification by sampling is proven generically with an arbitrary base-mechanism $\cM$ under the privacy-profile. The result, together with Theorem~\ref{thm:composition}, allows PLD formalism to be applied to a broader family of mechanisms as well as their subsampled versions under adaptive composition.

\subsection{Characteristic function representation}
%
%

Having strengthened the foundation of the PLD formalism with  ``dominating distribution pairs'' and two of its basic primitives, we can now put away RDP and its lossy $(\epsilon,\delta)$-DP conversion, then conduct mechanism-specific accounting under $(\epsilon,\delta)$-DP directly. Existing computational tools however, either require asymptotic approximation \citep{dong2019gaussian,sommer2019privacy}, repeated convolution \citep{dong2019gaussian} or an \emph{a priori} discretization of the output space \citep{koskela2020tight}.  This prompts us to ask:
\begin{center}
	\textsf{
``Can we compose mechanisms (with known dominating pairs) naturally just like in RDP? ''}
\end{center}
To achieve this goal, we propose using the  \textit{characteristic function} of the privacy loss RV. 
\begin{definition}[characteristic function of the privacy loss RV]\label{def: phi}
	Let $(P,Q)$ be a dominating pair of $\cM$, and $p,q$ be the probability density (or mass) function of $P,Q$. 
	The two characteristic functions that describes the PLD are 
	\begin{align*}
	\phi_{\cM}(\alpha):&= \E_P[e^{i\alpha \log(p/q)}], \;
	\phi_{\cM}'(\alpha):= \E_Q[e^{i\alpha \log(q/p)}],
	\end{align*}
	where $i$ denotes the imaginary unit satisfying $i^2 = -1$ and $\alpha \in \R$.
\end{definition}
PLDs are probability measures on the real line, and these $\phi$-functions are Fourier transforms of these measures. We provide $\phi$-functions  for basic mechanisms (see Table~\ref{tab: summary}) and the discrete mechanisms with closed-form expression. For other intricate and continuous mechanisms (e.g., subsample variants), we provide efficient discretization methods with error analysis in Section~\ref{sec: experiment}.

\noindent\textbf{Advantages over MGF}\label{par:advanrages_over_mgf}
Comparing to the moment generating function used by the RDP, the characteristic function differs only in that we are taking the expectation of the \textit{complex} exponential. At the price of bringing in complex arithmetics, it  is now a complex-valued function supported on $\alpha \in \mathbb{R}$ rather than the real-valued Renyi Divergence with order $\alpha>1$ as was defined in RDP. Unlike MGF, the characteristic function always exists and it characterizes the distribution of the privacy loss R.V., therefore it is always a \emph{lossless} representation. MGF is also characteristic when it exists, but the conversion of MGF to the distribution function is numerically problematic~\citep{epstein2008bad}.

Moreover,  the adaptive composition over multiple heterogeneous mechanisms remains as straightforward as that of the RDP.
\begin{proposition}\label{prop:composition}
	Let $\cM_1$ and $\cM_2$ be two randomized algorithms. We have the $\phi$-function of the composition $(\cM_1, \cM_2)$ with order $\alpha\in \R$ satisfies:
	$
	\phi_{(\cM_1, \cM_2)}(\alpha) = \phi_{\cM_1}(\alpha) \cdot\phi_{\cM_2}(\alpha)
	$
\end{proposition}

\begin{table*}
	\centering
	\centerline{
		\resizebox{1\textwidth}{!}{
			\tabcolsep=0.001cm
			\begin{tabular}{c@{\hskip 4mm}|c|@{\hskip 6mm}c@{\hskip 6mm}}
				\toprule
				Mechanism & Dominating Pair  &  $\phi$ function   \\
				\hline
				Randomized Response &$P: \pr_P[0]=p; Q:  \pr_Q[1]=p$ &$\phi_\cM(\alpha)=\phi'_\cM(\alpha)=pe^{\alpha i\log(\frac{p}{1-p})}+(1-p)e^{\alpha i \log(\frac{1-p}{p})}$\\
				Laplace Mechanism  &$P: p(x)=\frac{1}{2\lambda}e^{-|x|/\lambda}; Q: q(x)=\frac{1}{2\lambda}e^{-|x-1|/\lambda}$&$\phi_\cM(\alpha)=\phi'_\cM(\alpha)=\frac{1}{2}\bigg(e^{\frac{\alpha i}{\lambda}} + e^{\frac{-\alpha i -1}{\lambda}} + \frac{1}{2\alpha i +1} (e^{\frac{\alpha i }{\lambda}} -e^{\frac{-\alpha i -1}{\lambda}})\bigg)$\\
				Gaussian Mechanism  &$P:\cN(1, \sigma^2);Q:\cN(0, \sigma^2)$&$\phi_\cM(\alpha)=\phi'_\cM(\alpha) = e^{\frac{-1}{2\sigma^2}(\alpha^2 - i\alpha)}$\\
				\hline
			\end{tabular}
	}}
	\caption{$\phi$ functions and dominating pairs for basic mechanisms. }
	\label{tab: summary}
\end{table*}



\noindent\textbf{Lossless conversion rules. }
The $\phi$-function can be losslessly converted back and forth with other representation such as the privacy-profile, tradeoff function, moment-generating function as well as the distribution function of the privacy loss RV.
%
The conversion rule with prominent interest is the conversion to $(\epsilon,\delta)$-DP. Specifically, for finding $\delta$ as a function of $\epsilon$ (i.e., privacy profile),  we invoke the fourth equivalent definition of $(\epsilon,\delta)$-DP in Lemma~\ref{lem:equivalent_defs}, which depends on the cumulative distribution function (CDF) of the privacy loss random variables $L_{P,Q}$ and $L_{Q,P}$. 
In Appendix~\ref{sec:conversion_results}, we establish that these CDFs  can be evaluated through an integration of $\phi$-functions via Levy's formula.
\begin{figure*}[t]
	\centering
	\includegraphics[width=0.98 \textwidth ]{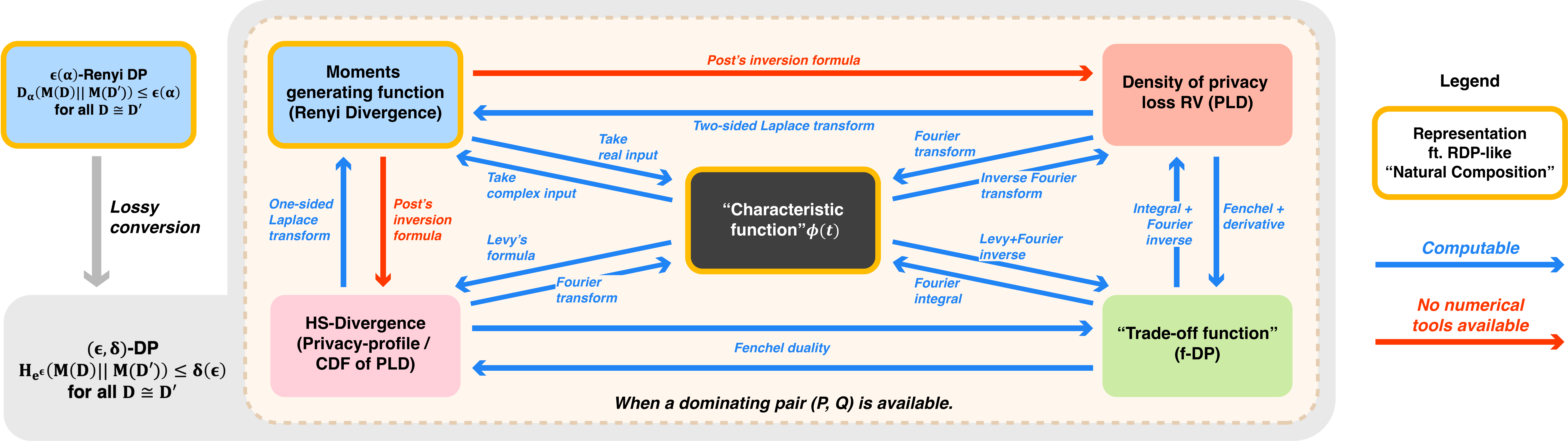}
	\caption{Summary of  the various functional descriptions and their conversion rules.}\label{fig:conversions}
	\vspace{-2mm}
\end{figure*}
The lossless conversions to other quantities are summarized in Figure~\ref{fig:conversions} and we provide more details in Appendix~\ref{sec:conversion_results}.  
  Moreover, most of the conversion formula correspond to well-known transforms such as the Fourier transform, Laplace transform and its double-sided variant. Except for those involve RDP and hence Laplace transform, numerical algorithms for implementing these transforms are often available.

\subsection{Analytical Fourier Accountant and numerical algorithms}

We now propose our analytical Fourier Accoutant (AFA) in Algorithm~\ref{alg: ours}, which is a combination of the \textit{lossless conversion rules} and the \textit{analytical composition rule} (Proposition~\ref{prop:composition}).
Given a sequence of mechanisms (can be varied) applied to the same dataset, the data structure tracks the $\log$ characteristic function of each mechanism in a symbolic form. When there is a $\delta(\epsilon)$ query, the accountant first constructs two analytical CDFs (with respect to the privacy loss RV $L_{P, Q}$ and $L_{Q, P}$) using Theorem~\ref{thm: levy} in Appendix~\ref{sec:conversion_results}.  Then the conversion to $(\epsilon, \delta)$-DP is obtained using Lemma~\ref{lem:equivalent_defs}. For computing $\epsilon$ given $\delta$, we use bisection to solve $\delta_\cM(\epsilon) = \delta$.
\begin{algorithm}[H]
	\caption{Analytical Fourier Accountant }
	\label{alg: ours}
		\small
	\begin{algorithmic}[1]
		\STATE \textbf{Input} Mechanisms $\cM_1, ..., \cM_K$  and  $\delta$ .\\
		\FOR{  $i=1, ..., K$}
		\STATE Maintain the symbolic accountant
		\STATE   $\log\phi_{(\cM)}(\alpha)\gets \log\phi_{(\cM)}(\alpha) + \log\phi_{(\cM_{i})}(\alpha)$
		\STATE   $\log\phi'_{(\cM)}(\alpha)\gets \log\phi'_{(\cM)}(\alpha) + \log\phi'_{(\cM_{i})}(\alpha)$
		\IF{ query $(\epsilon,\delta)$-DP}
		\STATE Compute the CDF  $F_{L_{P,Q}}(\cdot)$ and $F_{L_{Q,P}}(\cdot)$  by integrating  $\log\phi_{(\cM)}(\alpha)$  and $\log\phi'_{(\cM)}(\alpha)$ using Theorem~\ref{thm: levy}. 
		\STATE Return $\delta$ by Lemma~\ref{lem:equivalent_defs}.   
		\ENDIF
		\ENDFOR
	\end{algorithmic}
\end{algorithm}

\noindent\textbf{AFA vs FFT.} Comparing to the FFT-accountant approach \citep{koskela2020computing,koskela2020tight,koskela2021computing}, our approach decouples representation and numerical computation. We do not make any approximation when tracking the mechanisms,  and use numerical computation only when converting to $(\epsilon,\delta)$-DP. This avoids the need for setting appropriate discretization parameters of FFT ahead of time before knowing which sequence $\cM_1, ..., \cM_K$ we will receive.  

\noindent\textbf{Gaussian quadrature} For fast and numerically stable evaluation of the CDF, we propose to use Gaussian quadrature which adaptively selects the intervals between interpolation points, rather than the FFT approach which requires equally spaced discretization. When we apply this approach to efficiently evaluate integral in computing CDFs, where the numerical error is often negligible, i.e., $O(10^{-13})$ for CDFs in our experiments, even if we only sample a few hundreds points.  
We defer a more detailed error analysis to Section~\ref{sec: experiment}.

\section{Experiments}
In this section, we conduct numerical experiments to illustrate the behaviors of our analytical Fourier Accountant. We will have three sets of experiments.
\begin{enumerate}[itemsep=0em]
	\item[Exp. 1] (Gaussian mechanism) We compare the privacy cost over compositions  between RDP accountant and AFA accountant on Gaussian mechanism.
	\item[Exp. 2] (Compositions of discrete and continuous mechanisms) We evaluate the Fourier accountant variants and RDP accountant on heterogeneous mechanisms.
	\item[Exp. 3] (Compositions over Poisson Subsample mechanisms) Comparison of our AFA with discretization-based $\phi$-function to the Fourier accountant (FA) and the RDP accountant.
\end{enumerate}
\begin{figure*}[t]
	\centering	
	\subfigure[Exp1 Gaussian mechanism]{
		\includegraphics[width=0.32\textwidth]{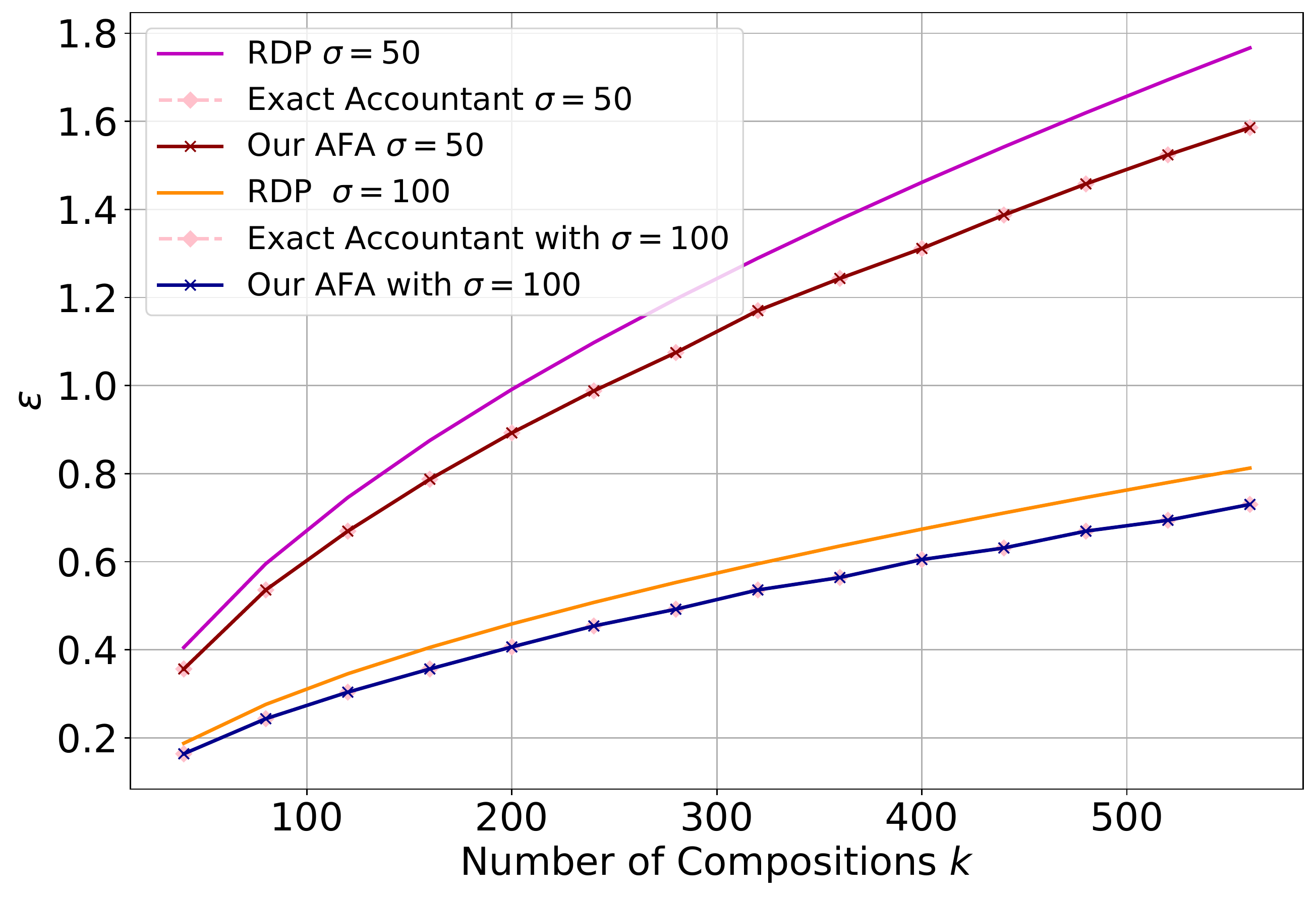}\label{exp1: b}}
	\subfigure[Exp2 heterogeneous mechanisms]{
		\includegraphics[width=0.32\textwidth]{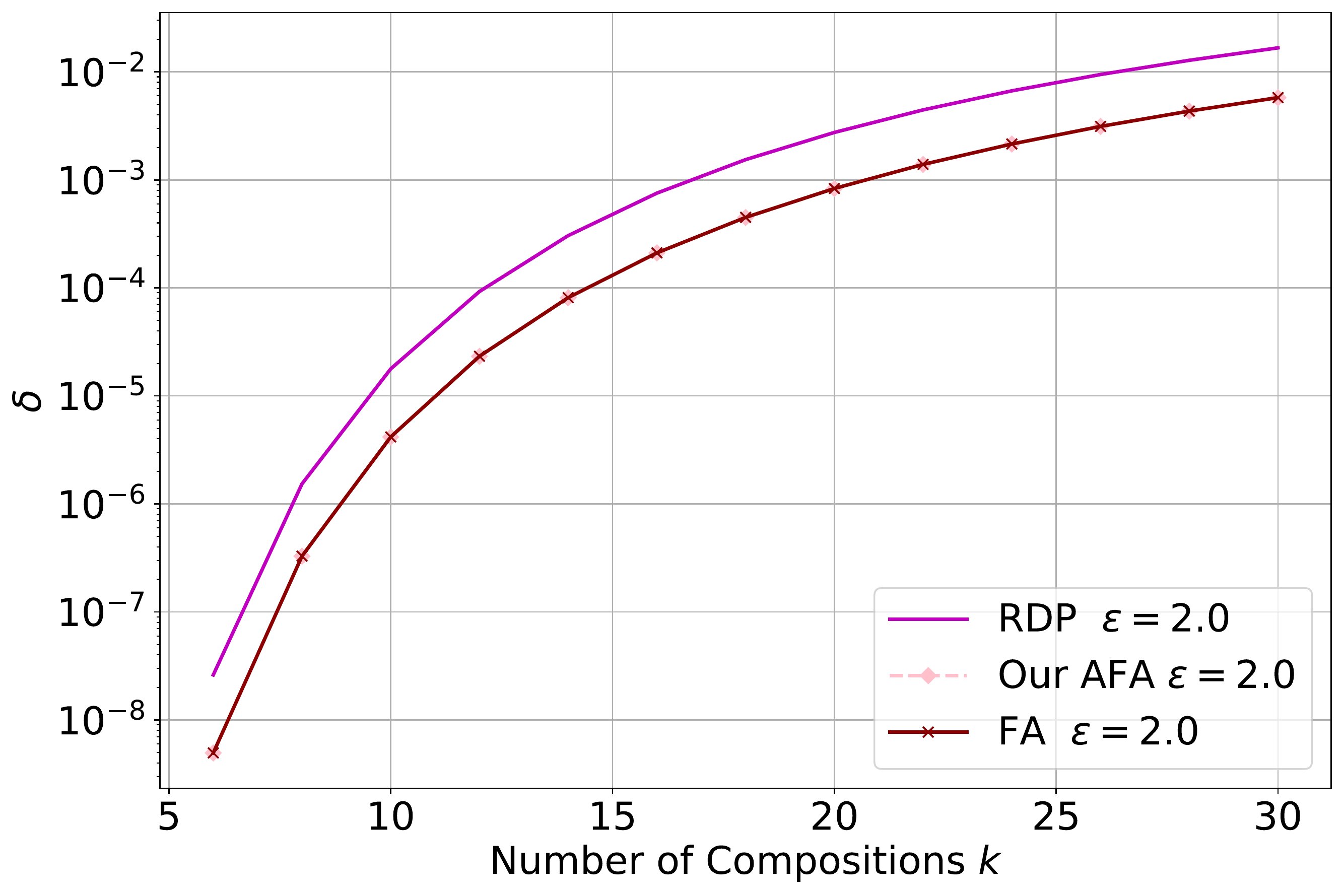}\label{fig: heter}}
	\subfigure[Exp3 Poisson Subsample ]{
		\includegraphics[width=0.32\textwidth]{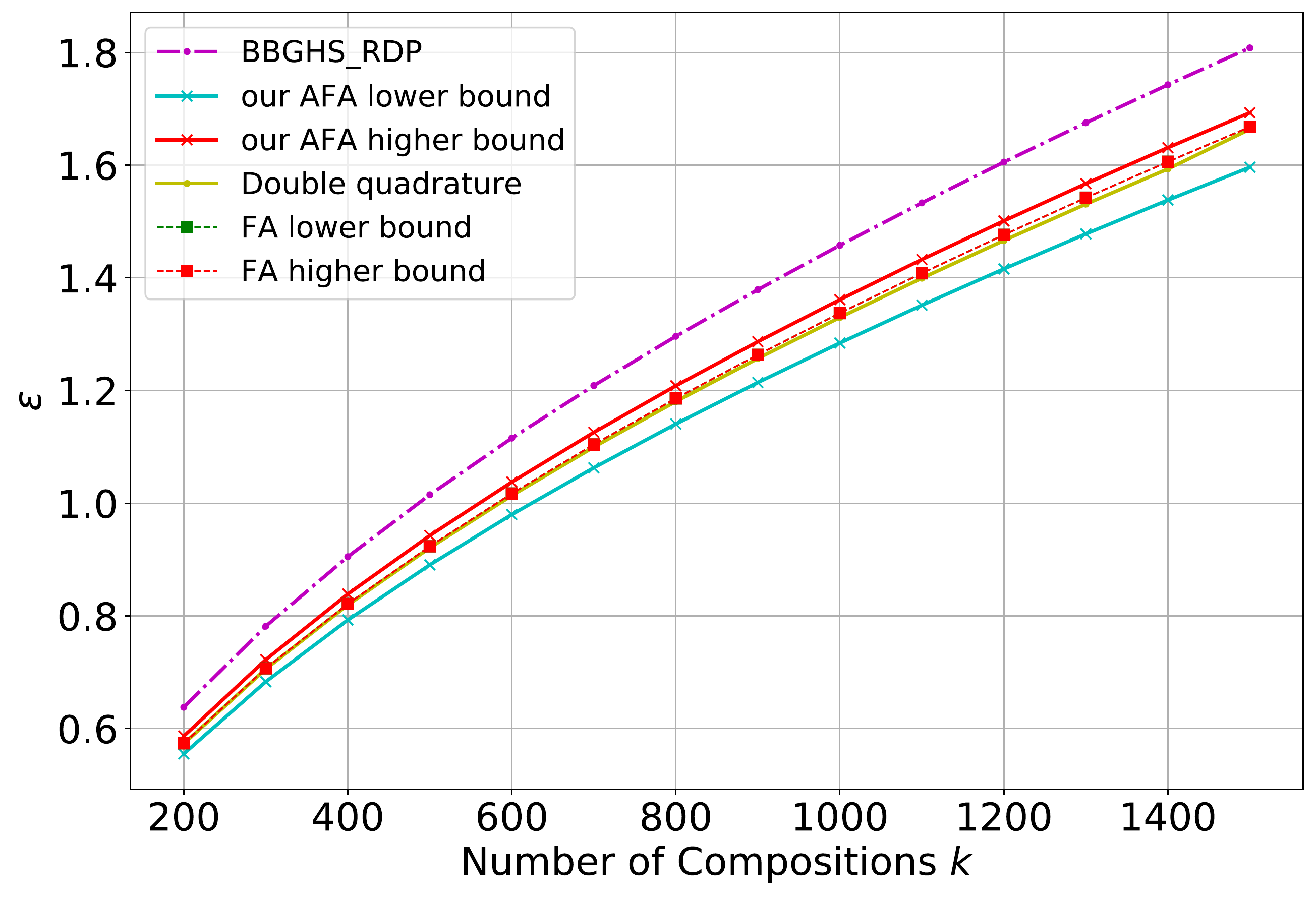}\label{fig: poisson}}
	\caption{ Pane~\ref{exp1: b} compares privacy cost over compositions in Exp 1. Pane~\ref{fig: heter} is for the heterogeneous composition in Exp 2. Pane~\ref{fig: poisson} is for Poisson subsampled Gaussian mechanism in Exp 3. } 
	\label{fig: exp_figure}
\end{figure*}

In Exp1, we compare our AFA method to the RDP-based accoutant\citep{mironov2017renyi} and the exact accountant from the analytical Gaussian mechanism~\citep{balle2018improving}. In Figure~\ref{exp1: b}, we evaluate $\epsilon$ with a fixed $\delta=10^{-4}$ and use $\sigma \in\{50, 100\}$.  

\noindent\textbf{Observation:} 
In Figure~\ref{exp1: b}, our $\phi$ function-based AFA exactly matches the result from the analytical Gaussian mechanism and strictly outperforms the RDP accountant in different privacy regimes.

In Exp2, motivated by~\citep{koskela2021computing}, we consider an adaptive composition of the form 
$\cM(X)= \bigg( \cM_1(X), \tilde{\cM}_2(X), ..., \cM_{k-1}(X), \tilde{\cM}_k(X)\bigg)$,
where each $\cM_i$ is a Gaussian mechanism with sensitivity $1$, and each $\tilde{\cM}_i$ is a randomized mechanism with probability $p$. We consider $\sigma=5.0, p=0.52, \epsilon=2.0$ and compare $\delta(\epsilon)$ between the RDP accountant, Fourier Accountant~\citep{koskela2021computing} and our AFA. 

Unlike the FA, our AFA  allows an analytical composition over discrete and continuous mechanisms without sampling discretisation points over the privacy loss distribution, therefore achieves an exact privacy accountant.  In Figure~\ref{fig: heter}, we plot the $\delta(\epsilon)$  over $k$ compositions given by FA and the moments accountant with RDP.  We use $n=10^5$ discretisations points and $L=10$ for FA. Our numerical result matches FA as $n=10^5$ is already a very accurate estimation as stated in~\citep{koskela2021computing}.

There are cases when the closed-form $\phi$-functions do not exist.   In Exp 3, we consider this problem by analyzing the Poisson Subsample Gaussian mechanism using our discretization-based approach (Algorithm~\ref{alg: appro}) and ``Double quadrature'' in Appendix~\ref{sec: experiment}.
We discuss the dominating distribution, the construction on $\phi$-function, and its discretization in Appendix~\ref{sec: experiment}.
Figure~\ref{fig: poisson} shows a comparsion of our AFA  to the Fourier accountant method~\citep{koskela2020tight} and the moments accountant method~\citep{zhu2019poission}. The sampling probability  is $\gamma=0.01$, the noise scale is $\sigma=2.0$ and we evaluate $\epsilon$ with $\delta=10^{-5}$.   
We use the tighter conversion rule from \citet{balle2020hypothesis} to convert the RDP back to $(\epsilon, \delta)$-DP.
The numerical issues induced by Gaussian quadrature are at most $O(10^{-14})$. Our lower and upper bounds of $\delta(\epsilon)$ shown in Figure~\ref{fig: poisson} already incorporate the error induced by discretization and ignoring the tail integral.
 We emphasize that the lower and upper bounds can match the bounds from FA by increasing sample points $n$. Moreover, ``Double quadrature'' is our proposed efficient approximation method. We only unevenly sample $700$ points for each $\phi$-function and the result of the ``Double quadrature'' lines between our lower and upper bounds and matches the result from FA. Lastly, all Fourier accountant-based approaches improve over the RDP-based accountant.


\noindent\textbf{Runtime and space analysis of AFA}
We first compare the time complexity and memory when we have analytical expressions of $\phi$-functions.  In Exp 2, each mechanism admits an analytical $\phi$-function and can be represented in $O(1)$ memory and evaluated in $O(1)$ time. Therefore, the memory cost is $O(\#$ unique mechanisms).  We analyze the runtime by decomposing it into the ``composition'' and ``conversion to $\delta(\epsilon)$'' separately.

Let $k$ denote  the number of compositions. Regarding the runtime in the conversion to $\delta(\epsilon)$ query, we apply Gaussian quadrature to compute the CDF, which requires $O(\frac{1}{\delta_{err}^{1/\alpha}})$ runtime complexity for the $\alpha$th order differentiable functions. The following composition runtime for \citet{koskela2021computing} and \citet{gopi2021numerical} denote the runtime for discretization and convolution via FFT for a homogeneous composition of a mechanism for $k$ rounds. We use $n$ to denote the size of grid discretization in the FFT approximation.
\begin{table*}[h]
	\centering
	\resizebox{\textwidth}{!}{ 
		\begin{tabular}{c|c|c|c|p{4.8cm}}
			\hline
			Privacy accountant & Composition runtime & $\delta(\epsilon)$ conversion runtime & Memory & Choice of $n$ \\
			\hline
			Our AFA & $O(1)$ & $O(\frac{1}{\delta_{err}^{1/\alpha}})$& $O(1)$ & Not applicable\\
			\citet{koskela2021computing} & $O(n\log n)$ & $O(n\log n)$& $O(n)$ & $n=O(k/\delta_{err})$\\
			\citet{gopi2021numerical} & $O(n\log n)$&$ O(n\log n)$ &$O(n)$ & $n=O(\sqrt{k}\log(1/\delta_{err})/\epsilon_{err})$\\
			\hline
		\end{tabular}
	}
	\caption{The runtime/space complexity   comparisons of different algoirthms}\label{tab:runtime}
\end{table*}

Of course, this is by no means a fair discussion because the FFT approach computes the entire (discretized) PLD of the composed mechanisms together while AFA computes just one point. In terms of the approximation error, our method is the only approach that adapts to the structures of the $\phi$ functions being integrated and achieves a faster convergence rate. 

For the cases when the analytical expressions of $\phi$-functions do not exist (see EXP3), we need to approximate the $\phi$ function too. Thus one single evaluation calls require $O(n)$, and our method is slower than \citet{koskela2021computing, gopi2021numerical}, because we do not use FFT. The space and time complexity of the adaptive discretization approach via double  quadrature is unclear, though very fast in practice.

\section{Conclusion}

In this paper, we studied the problem of privacy accounting with mechanism-specific analysis. We introduced the notion of \emph{dominating pair distributions}, showed that each mechanism's privacy profile is \emph{characterized} by a tight dominating pair, and derived a number of useful algebra of dominating pairs including \emph{adaptive composition} and \emph{amplification by sampling}. These results strengthen the foundation of the PLD formalism and make it more widely applicable. Algorithmically, we proposed an analytical Fourier accountant that represents the characteristic functions of a dominating pair \emph{symbolically}, which features RDP-like natural composition and allows us to leverage off-the-shelf numerical tools. Our experiments demonstrate the merits of AFA and suggest that it can flexibly and efficiently fit into every DP application.

This work also leaves several open questions.  Among those
\begin{itemize}
	\item As Lemma~\ref{prop:whatisHS} demonstrates, the construction of the domaining pair is severely constrained when trade-off functions are not clear. 
 For example, characterizing high-dimension discrete Gaussian mechanism remains a tricky open problem. 
 \item Moreover, there are cases where our approach requires much more quadrature points: We apply Gaussian quadrature to compute the CDF of the privacy loss RV through integration over $\phi$-functions. If the composed $\phi$ functions have large values at the tail of integral (e.g., near $\infty$), we need to sample more quadrature points. We hope to solve this issue using numerical tools in the next step.
\end{itemize}

\section*{Acknowledgments}
The work was partially supported by NSF CAREER Award \# 2048091,  Google Research Scholar Award and a gift from NEC Labs.  The authors thank the anonymous reviewers for catching a subtle issue in defining dominating pairs for $\alpha \geq 1$ \emph{only} in an earlier version of the paper. In the hindsight, defining $\alpha \geq 0$ is more natural and elegant.  We thank Antti Koskela and Thomas Steinke for helpful discussion.
We also thank Salil Vadhan for sharing a shorter alternative proof of the composition theorem based on a deep result due to Blackwell.

\bibliographystyle{plainnat}
\bibliography{DP}


\newpage
\onecolumn
\tableofcontents

\newpage
\onecolumn
\tableofcontents

\appendix



\section{Limits of RDP and the PLD formalism}\label{sec:limits_rdp}

In Section~\ref{sec: lossy_rdp} we omitted a few examples when we talk about the limitation of Renyi Differential Privacy (RDP) in describing common mechanisms. Specifically, we commented that there are mechanisms where RDP either does not exist or does not exist for most order $\alpha$ that implies stronger privacy guarantees. 

We give two concrete examples below.
\begin{example}[Distance-to-Instability]\label{ex:dist2stability}
	The stability-based argument of query release first add noise to a special integer-valued function $\mathrm{dist}_{q}(D)$ which measures the number of data points to add / remove before the local sensitivity of query $q(D)$ becomes non-zero.  No matter that $q$ is, $\mathrm{dist}_{q}$ always has a global sensitivity of at most $1$.  The stability-based query release outputs $\perp$ (nothing) if  $\mathrm{dist}_{q}(D) +\mathrm{Lap}(1/\epsilon) \leq \nicefrac{\log(1/\delta)}{\epsilon}$ otherwise outputs the answer $q(D)$ without adding noise.  This algorithm is satisfies $(\epsilon,\delta)$-DP \citep{thakurta2013differentially}, but since there is a probability mass at the $+\infty$ for the case when $q(D)\neq q(D')$, RDP is $+\infty$ for all $\alpha$.
\end{example}
\begin{example}[Gaussian-noise adding with data-dependent variance]
	In smooth sensitivity-based query release \citep{nissim2007smooth}, one perturbs the output with a noise with a data-dependent variance.  
	Consider, for example, $P = \cN(0,\sigma_1^2), Q = \cN(0,\sigma_2^2)$, then the Renyi-divergence $D_\alpha(P\|Q)$  is undefined for all $\alpha$ such that $\alpha\sigma_2^2 + (1-\alpha)\sigma_1^2< 0$. Specifically, if $\sigma_1^2 = 2, \sigma_2^2 = 1$, then $D_\alpha(P\|Q)=+\infty$ for all $\alpha\geq 2$.
	%
\end{example}
These examples demonstrate the deficiency of RDP in analyzing flexible algorithm design tools such as the proposed-test-release~\citep{dwork2009differential}), which typically introduces a heavier-tailed privacy-loss distributions for which the moment generating function is not defined.  

On the contrary, the privacy-profile is well-defined in both examples and imply nontrivial $(\epsilon,\delta)$-DP. The characteristic function exists no matter how heavy-tailed the distribution of the privacy loss random variable is so it naturally handles the second example. In Section~\ref{sec:mass_at_inf}, we describe how we can handle a probability mass at $+\inf$ in our approach.

We also omitted an example for which there are no single pair of neighboring datasets that attain the argmax might be different in different regions of the privacy profile. 
\begin{example}[Distance to Instability]\label{ex:worst_case_pair_inputs_do_not_exist}
	Distance to instability $\mathrm{dist}_{q}(D)$ is a special function that measures the number of data points to add / remove before the local sensitivity of query $q(D)$ becomes non-zero.  The stability-based query release outputs $\perp$ (nothing) if  $\mathrm{dist}_{q}(D) +\mathrm{Lap}(1/\epsilon) \leq \nicefrac{\log(1/\delta)}{\epsilon}$ otherwise outputs the answer $q(D)$ without adding noise.  
	In this algorithm, the privacy loss distribution has exactly two modes. 
	\begin{description}
		\item[Mode 1]  When  $\mathrm{dist}_{q}(D)>0$, then for all $D'$ neighboring to $D$, $q(D)=q(D')$, which implies that the PLD is from the post-processing of a Laplace mechanism (for releasing the perturbed $\mathrm{dist}_{q}(D)$), i.e., $(\epsilon,0)$-DP.
		\item[Mode 2] When $\mathrm{dist}_{q}(D)=0$, then for those neighboring $D'$ such that $q(D)\neq q(D')$, it must hold that $\mathrm{dist}_{q}(D')=0$, thus the privacy loss distribution is a point mass of $1-\delta$ at $0$ (for outputting $\perp$) and a point mass of $\delta$ at $+\infty$, i.e., $(0,\delta)$-DP. 
	\end{description}
	Clearly, there is no single pair of datasets that attains the privacy-profile of this mechanism for all input parameter $\tilde{\epsilon}$. When $\tilde{\epsilon}>\epsilon$, $\delta_{\cM}(\tilde{\epsilon}) = \delta$ and is attained by the second mode. On the other hand, if we choose $\tilde{\epsilon}$ such that $\delta_{\textrm{Lap. Mech.}(1/\epsilon)}(\tilde{\epsilon}) > \delta$, then $\delta_{\cM}(\tilde{\epsilon}) =\delta_{\textrm{Lap. Mech.}(1/\epsilon)}(\tilde{\epsilon})$ and the equal sign is attained by a pair of distributions in the first mode.
\end{example}


 \section{Conversion rules between functional representations} 
 \label{sec:conversion_results}
In this section we give the details of conversions between various functional representations of the privacy loss distribution (of a dominating pair of distributions). These conversions are summarized in \Cref{fig:conversions} and repeated here.

\begin{figure}[h]
	\centering
	\includegraphics[width= 0.9\textwidth ]{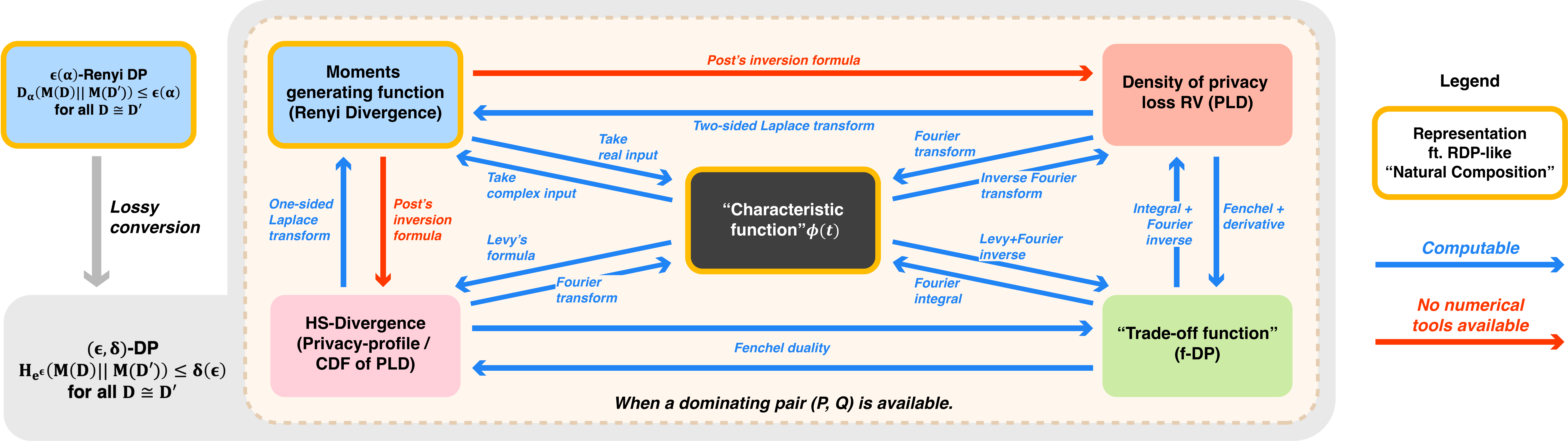}
\end{figure}

Before we proceed to the details of all these arrows, we would like to emphasize a important distinction:
\begin{quote}
	\textsf{These conversions rules are \emph{not} about converting between different DP definitions, but rather converting between different representations of the privacy loss r.v. under the same DP definition --- in our case, $(\epsilon,\delta)$-DP.}
\end{quote}
More precisely, we mean that the conversion from RDP to DP (leftmost grey arrow in the figure, which we will talk about in details in Section~\ref{sec:rdp2dp})  is \emph{qualitatively different} from the conversion from Renyi divergence to Hockey-Stick divergence (red arrow labeled ``Post's inversion formula'').

Modulo some details\footnote{such as the symmetry of $P=M(D)$ and $Q=M(D')$ and the domains of $\alpha$ and $\eps$.}, a conversion from RDP to DP is about finding function $\delta(\cdot)$ that upper bounds the Hockey-stick divergence for all pairs of neighboring datasets using an RDP function $\epsilon(\cdot)$.
\begin{center}
	If $\sup_{D\simeq D'}\cD_\alpha(\cM(D)\|\cM(D'))\leqslant \epsilon(\alpha)$, then $\sup_{D\simeq D'} H_{\e^\eps}(\cM(D)\|\cM(D'))\leqslant \delta(\eps)$.
\end{center}
In contrast, a conversion from Renyi divergence to hockey-stick divergence is about a given pair of $P,Q$, and the input function $\epsilon(\alpha)$ is expected to be the exact Renyi-divergence of order $\alpha$. The goal of the divergence-to-divergence conversion rule is to find a different divergence of the same pair of distribuiton $P,Q$, i.e. 
\begin{center}
	If $\cD_\alpha(P\|Q)= \epsilon(\alpha)$, then $H_{\e^\eps}(P\|Q)= \delta(\eps)$.
\end{center}
Both conversions aim to compute a function $\delta$ from a function $\epsilon$. The seemingly harmless distinction of inequalities and identities is actually the devil in the details. It has two major consequences
\begin{enumerate}
	\item When applied to privacy, divergence conversion requires a dominating pair of distributions as a prerequisite, which may or may not be a tight dominating pair. In the figure, results that require a dominating pair are enclosed in the light yellow region labeled ``When a dominating pair $(P,Q)$ is available''.
	\item DP conversion is lossy even when converting the statement ``standard randomized response is 1-zCDP'' to $(\eps,\delta)$-DP, as demonstrated by \Cref{fig: lossy_rdp}. On the other hand, divergence conversion is generically lossless (under some regularity condition), though numerical issues often arise since the inverse Laplace transform is involved \cite{epstein2008bad}.
\end{enumerate}
In alignment with the focus of this paper, in this section we focus on the light yellow region assuming $(P,Q)$ is a dominating pair. DP conversion is discussed in more detail in \Cref{sec:rdp2dp}.

\begin{table}[H]
\centering
\begin{tabular}{|c|c|m{85mm}|}
\hline
From & To & Result
\\ \hline
$\phi, \phi'$ & $F,G$ & \Cref{thm: levy}, a direct consequence of Levy's formula\\ \hline
$F,G$ & $\phi, \phi'$ & Fourier transform, by definition\\ \hline
$F,G$ & $H_\alpha$ & \Cref{lem:HSfromFG}\\ \hline
$H_\alpha$ & $F,G$ & \Cref{lem:FGfromHS}\\ \hline
$F,G$ & $f$ & \Cref{lem:ffromFG}\\ \hline
$f$ & $F,G$ & \Cref{lem:FGfromHS}\\ \hline
$H_\alpha$ & $f$ & Proposition 2.12 of \cite{dong2019gaussian}, restated as \Cref{lem:ffromHS}\\ \hline
$f$ & $H_\alpha$ & Proposition 2.12 of \cite{dong2019gaussian}, restated as \Cref{lem:HSfromf}\\ \hline
 $H_\alpha$ & $\cD_\alpha$
 & Theorem 8 of \cite{balle2018couplings}, restated as \Cref{lem:restateBorja}\\ \hline
$\cD_\alpha$ & $H_\alpha$ & Post's formula. In fact, any inverse Laplace transform works. \\ \hline
$\cD_\alpha$ & $\phi$ & take pure imaginary input. Need analytic extension in general and not always possible.\\ \hline
$\phi, \phi'$ & $\cD_\alpha$ & take pure imaginary input. Need analytic extension in general and not always possible.\\ \hline
$f$ & $\phi, \phi'$ & \Cref{lem:f.ch.f}\\ \hline
$\phi, \phi'$ & $f$ & first use Levy's formula to compute $F$ and $G$, then use \Cref{lem:ffromFG}\\ \hline
\end{tabular}
\caption{References for conversion results declared in \Cref{fig:conversions}. Notations: ch.f. of PLD are denoted by $\phi, \phi'$. CDFs of PLD are denoted by $F,G$. The trade-off function is denoted by $f$ (note that $f$ is \emph{not} the derivative of $F$). Hockey-stick divergences are denoted by $H_\alpha$. Renyi divergences are denoted by $\cD_\alpha$.}
\end{table}

We recall some definitions. Let $P,Q$ be two probability distributions on the same measurable space. For $\alpha>0$, their hockey-stick divergence is defined as
	$$H_\alpha(P\|Q) = \E_{\omega\sim Q}[ (\tfrac{\diff P}{\diff Q}(\omega) - \alpha)_+ ].$$
For $\alpha>1$, their Renyi divergence is defined as
	\begin{align*}
&\cD_{\alpha}(P\|Q)
:= \tfrac{1}{\alpha-1}\log \mathbb{E}_P \left( \tfrac{\diff P}{\diff Q}\right)^\alpha .
	\end{align*}
Let $F$ and $G$ be the CDFs of the privacy loss random variables. Namely,
\begin{align*}
	F(x):&=P[\log\tfrac{\diff Q}{\diff P}\leqslant x]\\
	G(x):&=Q[\log\tfrac{\diff Q}{\diff P}\leqslant x]
\end{align*}
The corresponding densities (if exist) will be $F'$ and $G'$. The corresponding characteristic functions (ch.f.) are the Fourier transforms of the two measures, i.e.
\begin{align*}
	\phi(t) &= \int \e^{it x}\diff F(x)\\
	\phi'(t) &= \int \e^{it x}\diff G(x)
\end{align*}
Trade-off functions are $T[P,Q]$ and $T[Q,P]$, which map the type I error to the corresponding minimal type II error in testing problems $P$ vs $Q$ and $Q$ vs $P$ respectively.

From these definitions we see that all five functional representations actually require \textbf{two} functions for each pair of distributions. Below we summarize how one determines the other.
\begin{itemize}
	\item $H_\alpha(Q\|P) = \alpha H_{\alpha^{-1}}(P\|Q)-\alpha+1$, which is stated as \Cref{lem:HSpq} in \Cref{sec:omitted_proofs_in_sec:conversion_results}.
	\item For $\alpha\in(0,1)$, $\cD_\alpha(Q\|P) = \frac{\alpha}{1-\alpha} \cD_{1-\alpha}(P\|Q).$
	See Proposition 2 of \cite{van2014renyi}.
	\item 
		$G'(x)= \e^{x} F'(x)$, which is stated as \Cref{lem:GfromF} in \Cref{sec:omitted_proofs_in_sec:conversion_results}.
	\item Using the above formula, $\phi'$ can be obtained by the following process: $\phi\xrightarrow{\text{Levy's formula}} F\to G \to\phi'$.
	\item If $T[P,Q]=f$ then $T[Q,P]=f^{-1}$.
	See Lemma A.2 of \cite{dong2019gaussian}
\end{itemize}

We now consider the conversion from the $\phi$-function to  CDFs using the following Levy's theorem.

\begin{theorem}[Levy]\label{lem: levy}
	Let $\phi$ be the ch.f. of the distribution function $F$ and $a < b$, then
	$$
	F(b)-F(a) = \lim_{T \to \infty}\frac{1}{2\pi} \int_{-T}^T \frac{\e^{-it a}-\e^{-it b}}{it}\cdot \phi(t)\diff t.
	$$
\end{theorem}
Note that $\lim_{T\to \infty, \alpha \to \infty} \int_{-\infty}^\infty \frac{e^{-i\alpha a}}{i\alpha}\phi(\alpha) d\alpha = \pi$. To compute the CDF of the privacy loss RV $L_{P,Q}$ at $b$, we can substitude $a$ with $-\infty$ and obtain the following result.

\begin{lemma}\label{thm: levy}
	\begin{align*}
		 F(x) &= \frac{1}{2} + \lim_{T \to \infty }\frac{1}{2\pi} \int_{-T}^T \frac{i\e^{-it x}}{t} \phi(t) \diff t\\
		 G(x) &= \frac{1}{2} + \lim_{T \to \infty }\frac{1}{2\pi} \int_{-T}^T \frac{i\e^{-it x}}{t} \phi'(t) \diff t
	\end{align*}
\end{lemma}

\begin{lemma}\label{lem:HSfromFG}
\begin{align*}
	H_\alpha(P\|Q) &= F(-\log\alpha)-\alpha G(-\log \alpha)\\
	H_{\alpha}(Q\|P) &= 1-G(\log \alpha)-\alpha(1-F(\log \alpha))
\end{align*}
\end{lemma}
\begin{lemma}\label{lem:HSfromf}
	$H_{\e^\eps}(Q\|P) = 1+f^*(-\e^\eps)$
\end{lemma}
\begin{lemma}\label{lem:ffromHS}
	$$
	f( x) = \sup_{\eps\geqslant 0} \max\{0, 1-H_{\e^\eps}(P\|Q) - \e^\eps  x, \e^{-\epsilon}(1-H_{\e^\eps}(P\|Q)  -  x)\}.
	$$
\end{lemma}

\begin{lemma}\label{lem:ffromFG}
	$f(\alpha) = G(F^{-1}(1-\alpha))$.
\end{lemma}
\begin{lemma}\label{lem:FGfromHS}
	\begin{align*}
		F(x) &= 1-(f^*)'(-\e^x)\\
		&=1+\e^{-x}\cdot\frac{\diff}{\diff x}H_{\e^x}(Q\|P)\\
		G(x) &= f\big((f^*)'(-\e^x)\big)\\
		&= 1-H_{\e^x}(Q\|P)+\frac{\diff}{\diff x}H_{\e^x}(Q\|P)
	\end{align*}
\end{lemma}
\begin{lemma}[Theorem 6 of \cite{balle2018couplings}]\label{lem:restateBorja}
	\begin{align*}
		\cD_{\alpha}(P\|Q)=\tfrac{1}{\alpha-1}\log \left(1+\alpha(\alpha+1)\int_0^\infty\left(\e^{\alpha\eps}H_{\e^\eps}(P\|Q)+\e^{-(\alpha+1)\eps}H_{\e^\eps}(Q\|P)\right)\diff \eps\right).
	\end{align*}
\end{lemma}
\begin{lemma}\label{lem:f.ch.f}
The characteristic functions are determined by the trade-off function $f$ via the following formula:
\begin{align*}
	\phi(t)&=\int_0^1\e^{-it\log|f'(x)|}\diff x\\
	\phi'(t)&=\int_0^1\e^{it\log|f'(x)|}\cdot|f'(x)|\diff x
\end{align*}
\end{lemma}



\section{Omitted proofs in the main body}\label{sec:proofs}

\subsection{Characterization of privacy profiles}

\begin{proof}[Proof of Proposition~\ref{prop:whatisHS}]
	Let
	\begin{align*}
	\cH&:=\{h:\R_{\geqslant0}\to\R_{\geqslant0}\mid \exists P,Q\text{ s.t. }h(\alpha)=H_\alpha(P\|Q) \},\\
	\cF&:=\{f:[0,1]\to[0,1]\mid \exists P,Q\text{ s.t. }f=T[P,Q] \}.
	\end{align*}
	
	
	By \Cref{lem:HSfromf}, $H_\alpha(P\|Q)$ can be related to $f=T[Q,P]$ as follows:
	$$H_{\e^\eps}(P\|Q) = 1+f^*(-\e^\eps)$$
	where $\eps$ ranges over the whole real line.
	By a simple change of variable, we see that
	$h\in\cH$ iff there exists $f\in\cF$ such that $h(\alpha)=1+f^*(-\alpha)$, or equivalently,
	$$
	\cH=\{h:\R_{\geqslant0}\to\R_{\geqslant0}\mid \exists f\in\cF, h(\alpha)=1+f^*(-\alpha)\}.
	$$
	By Proposition 2.2 of \citet{dong2019gaussian}, we know
	\[
	\cF=\{f:[0,1]\to[0,1]\mid f\text{ is convex, decreasing, continuous and }f(x)\leqslant 1-x \}.
	\]
	Let $\cG:=\{g:(-\infty,0]\to\R\mid g(0)=0,g\text{ is convex, increasing, continuous and }g(x)\geqslant \max\{x,-1\} \}$.
	\begin{claim}
		Convex conjugacy is a bijection between $\cF$ and $\cG$.
	\end{claim}
	\begin{proof}[Proof of the claim]
		Since both $\cF$ and $\cG$ consist of convex functions, double convex conjugacy brings back the function, it suffices to show that $f\in\cF\implies f^*\in\cG$ and $g\in\cG\implies g^*\in\cF$. 
		Now suppose $f\in\cF$. $f$ is extended to be $+\infty$ in $(-\infty,0)$ and 0 in $(1,+\infty)$. Thus $f$ is a convex function on $\R$. By definition $f^*$ is convex, and we can calculate
		\[
		f^*(y)=\sup_{x\in\R} yx-f(x)=\sup_{x\geqslant 0} yx-f(x) = \left\{
		\begin{aligned}
		+\infty, \text{ if } y>0\\
		0, \text{ if } y=0
		\end{aligned}
		\right.
		\]
		With $y_1<y_2$, we have $y_1 x-f(x)\leqslant y_2 x-f(x)$. Taking supremum over $x\geqslant 0$, we have $f^*(y_1)\leqslant f^*(y_2)$. This shows $f^*$ is monotone and finite on $(-\infty,0]$. Let 
		$$I(x)=\left\{
		\begin{aligned}
		+\infty, \text{ if } x<0\\
		\max\{1-x,0\}, \text{ if } x\geqslant0
		\end{aligned}
		\right.$$ It is straightforward to compute that
		$$I^*(y) = \left\{
		\begin{aligned}
		\max\{y,-1\}, \text{ if } y\leqslant0\\
		+\infty, \text{ if } y>0\\
		\end{aligned}
		\right.$$
		Since $f\leqslant I$, we conclude that $f^*(x)\geqslant I^*(x)= \max\{x,-1\}$.
		
		Now suppose $g\in\cG$. Similarly, $g$ is extended to be $+\infty$ in $(0,+\infty)$. $g^*(y)=\sup_{x\leqslant 0}yx-g(x)$ and $g^*(y)=+\infty$ if $y<0$. By a similar argument, $g^*$ is increasing. Since $g\geqslant I^*$, we have $g^*\leqslant I^{**} = I$. That is, $g^*(x)\leqslant 1-x$. Let $J$ be zero on $(-\infty,0]$ and infinity otherwise. We have $J^*$ is zero on $[0,+\infty)$ and infinity otherwise. We know that $g(0)=0$ and $g$ is increasing so $g\leqslant J$. Hence $g^*\geqslant J^*$, i.e. $g^*(y)\geqslant 0$ if $y\geqslant 0$. This justifies that $g^*(y)\in[0,1]$ if $y\in[0,1]$ and $g^*(y)=0$ if $y\geqslant 1$.
	\end{proof}
	
	Now with the help of this claim, $\cH$ and $\cG$ are simply related: $h\in\cH$ iff $\alpha\mapsto h(-\alpha)-1$ is in $\cG$. Therefore we can get the description of $\cH$. The proof of the first statement is complete.		
	
	\noindent\textbf{Explicit construction.} Next we derive the specific choice of $P,Q$ as stated works using the result from \cite{dong2019gaussian}.
	
		Continuing with the notations in the proof above, when $H$ satisfies the conditions, i.e. $H\in\cH$, we know there is a $f\in\cF$ such that $H(\alpha)=1+f^*(-\alpha)$. Let $g(\alpha)=H(-\alpha)-1$ and we will have $g=f^*$ and hence $f=g^*$ as $f$ is convex.
	Therefore,
	$$f(x) = g^*(x)=\sup_y yx - H(-y)+1 = \sup_z -zx - H(z)+1 = 1+H^*(-x).$$
	From \citet[Proposition 2.2]{dong2019gaussian}, we know that $f=T[Q,P]$ where $Q=U[0,1]$ is the uniform distribution over $[0,1]$ and $P$ has CDF
	$$F_P(x)=
	\begin{cases}
	0, \text{ if } x<0,\\
	f(1-x), \text{ if } x\in[0,1),\\
	1, \text{ if } x\geqslant1.\\
	\end{cases}$$
	Plugging in $f(x) =1+H^*(-x)$, we have the CDF of $P$ being
	$$F_P(x)=
	\begin{cases}
	0, \text{ if } x<0,\\
	1+H^*(1-x), \text{ if } x\in[0,1),\\
	1, \text{ if } x\geqslant1.\\
	\end{cases}$$
	Note that when the infimum of $H$ is positive, $H^*(1-x)<0$ and $P$ has an atom at 1.
	This completes the proof.
\end{proof}
%

Another interesting consequence of \Cref{prop:whatisHS} is one can often get a stronger bound on the hockey-stick divergence or privacy profile \emph{for free}.
Recall 
that for a 
function $g$, its convex hull $\mathrm{conv}(g)$ (a.k.a., the lower convex envelope) is defined as 
the greatest convex lower bound of $g$ and satisfies $\mathrm{conv}(g)=g^{**}$ where the double star means taking Fenchel conjugate twice.

For a function $h: \R_+ \rightarrow \R$, let $g(x)=\inf_{y\in[0,x]} h(y)$ and $\mathrm{HS}(h) = (\min\{1,g\} )^{**}$. It turns out that $\mathrm{HS}(h)$ is the greatest lower bound of $h$ that lies in $\cH$, and we have
\begin{corollary}[Dominating pairs from any privacy profile upper bounds] \label{cor:dominating_pair}
	If the privacy profile of a mechanism $\cM$ is bounded by $h: \R_+ \rightarrow \R$, i.e. $\delta_\cM(\alpha)\leqslant h(\alpha),\forall \alpha\geqslant 0$,
	then $\delta_\cM$ is also bounded by $\mathrm{HS}(h)$.	
\end{corollary}
Note that $\mathrm{HS}(h)$ can be significantly smaller than the original bound $h$, and it admits a dominating pair by Proposition~\ref{prop:whatisHS}, even if $h$ does not.

\begin{proof}
	We know that $\delta_\cM\in\cH$. It suffices to show that
	\[f\in\cH,f\leqslant h\implies f\leqslant \mathrm{HS}(h).\]
	Recall that we let $g(x)=\inf_{y\in[0,x]} h(y)$ and $\mathrm{HS}(h) = (\min\{1,g\} )^{**}$. Since $f\in\cH$ is decreasing, $f(x)= \inf_{y\in[0,x]} f(y)\leqslant \inf_{y\in[0,x]} h(y) = g(x)$. Furthermore, $f(x)\leqslant f(0)=1$, so $f\leqslant \min\{1,g\}$. Since $f$ is convex, it also holds that $f\leqslant \min\{1,g\}^{**}=\mathrm{HS}(h)$.
\end{proof}

\subsection{Composition theorem of dominating pairs}

\begin{theorem}[Restatement of Theorem~\ref{thm:composition} Adaptive composition of dominating pairs]
	Let $P,Q$ be a dominating pair distributions for $\cM$ and $P', Q'$ be a dominating pair distributions for $\cM'$\footnote{$\cM'$ can be adaptively chosen in that it could depend on the output of $\cM$, which requires $\sup_{o\in \textrm{Range}(\cM)}H_{e^\epsilon}(\cM'(D,o)\| \cM'(D',o)) \leq H_{e^\epsilon}( P'\|Q')$ for any value of $o$. 
	},  then 
	$
	(P\times P', Q\times Q')
	$
	is a dominating pair distributions for the composed mechanism $(\cM, \cM')$.
\end{theorem}
\newcommand{\HA}{H_\alpha}
\begin{proof}
	\[\HA(P\|P')=\int_\Omega [p(\omega)-\alpha p'(\omega)]_+\diff \omega.\]
	Integration with respect to a dominating measure of both $P$ and $Q$ and $p,q$ are the densities (Radon-Nikodym derivatives) for the probability measures $P,Q$ respectively.

	Our goal is to show $\HA\big(M(D),M(D')\big)\leqslant\HA\big(P\times R,Q \times S\big)$. We break it into the following two parts.
	\[\HA\big(M(D),M(D')\big)\leqslant \HA\big(M_1(D)\times R,M_1(D')\times S\big)\leqslant\HA\big(P\times R,Q \times S\big).
	\]
	Starting from the first part, we have
	\begin{align*}
		\HA\big(M(D),M(D')\big)
		&=\iint_{X\times Y} [p_1(x)p_2(x,y)-\alpha p_1'(x)p_2'(x,y)]_+\diff x\diff y\\
		&=\int_{X} p_1(x) \cdot\left(\int_Y \left[p_2(x,y)-\alpha\cdot\frac{p_1'(x)}{p_1(x)}\cdot p_2'(x,y)\right]_+\diff y\right) \diff x\\
		&=\int_{X} p_1(x) \cdot\Big(H_{\alpha\cdot\frac{p_1'(x)}{p_1(x)}}\big(M_2(D,x)\|M_2(D',x) \big)\Big) \diff x\\
		&\leqslant\int_{X} p_1(x) \cdot\Big(H_{\alpha\cdot\frac{p_1'(x)}{p_1(x)}}\big(R\|S \big)\Big) \diff x\\
		&=\int_{X} p_1(x) \cdot\left(\int_{\Omega_2} \left[r(\omega_2)-\alpha\cdot\frac{p_1'(x)}{p_1(x)}\cdot s(\omega_2)\right]_+\diff \omega_2\right) \diff x\\
		&=\iint_{X\times \Omega_2} [p_1(x)r(\omega_2)-\alpha p_1'(x)s(\omega_2)]_+\diff x\diff \omega_2\\
		&=\HA\big(M_1(D)\times R,M_1(D') \times S\big).
	\end{align*}
	Continuing this argument, we have
	\begin{align*}
		\HA\big(M_1(D)\times R,M_1(D') \times S\big)
		&=\iint_{X\times \Omega_2} [p_1(x)r(\omega_2)-\alpha p_1'(x)s(\omega_2)]_+\diff x\diff \omega_2\\
		&=\int_{\Omega_2} r(\omega_2) \cdot\left(\int_{X} \left[p_1(x)-\alpha\cdot\frac{s(\omega_2)}{r(\omega_2)}\cdot p_1'(x)\right]_+\diff x\right) \diff \omega_2\\
		&=\int_{\Omega_2} r(\omega_2) \cdot\Big(H_{\alpha\cdot\frac{s(\omega_2)}{r(\omega_2)}}\big(M_1(D)\|M_1(D') \big)\Big) \diff \omega_2\\
		&\leqslant\int_{\Omega_2} r(\omega_2) \cdot\Big(H_{\alpha\cdot\frac{s(\omega_2)}{r(\omega_2)}}\big(P\|Q\big)\Big) \diff \omega_2\\	
		&=\int_{\Omega_2} r(\omega_2) \cdot\left(\int_{X} \left[p(\omega_1)-\alpha\cdot\frac{s(\omega_2)}{r(\omega_2)}\cdot q(\omega_1)\right]_+\diff \omega_1\right) \diff \omega_2\\
		&=\iint_{\Omega_1\times \Omega_2} [p(\omega_1)r(\omega_2)-\alpha q(\omega_1)s(\omega_2)]_+\diff \omega_1\diff \omega_2\\
		&=\HA\big(P\times R,Q \times S\big).
	\end{align*}
	The proof is complete.
\end{proof}



\subsection{Privacy-amplification for dominating pairs}
Recall we stated the following theorem in the main body:
\amp*

\newcommand{\Id}{\mathrm{Id}}
\newcommand{\unif}{\textbf{\textup{Subset}}}
\newcommand{\pois}{\textbf{\textup{Poisson}}}
The proof we present here is written in the language of trade-off functions \citep{dong2019gaussian}. However, with \Cref{lem:HSfromf}, everything can be conveniently translated to the language of $(\varepsilon,\delta)$. We made the choice because some parameters have slightly easier forms in the language of trade-off functions.

	
We begin with a lemma that cut our workload in half --- dominance for removal neighbors is actually equivalent to the dominance of add neighbors, so it suffices to show either one of them.
\begin{lemma}\label{lem:primal}
The followings are equivalent
\begin{enumerate}
	\item 
	$(P,Q)$ dominates mechanism $\cM$ for add neighbors.
	\item 
	$(Q,P)$ dominates mechanism $\cM$ for removal neighbors.
	\item $T[M(S), M(S\cup\{x\})]\geqslant T[P,Q]$ for any dataset $S$ and data entry $x$.
\end{enumerate}
\end{lemma}
\begin{proof}[Proof of \Cref{lem:primal}]
	Recall that \Cref{lem:HSfromf} says
	for any $\alpha>0$,
	\[\HA(P\|Q)=1+T[P,Q]^*(-\alpha).\]
	\begin{align*}
		\text{condition 1}
		&\Leftrightarrow\HA(M(S)\|M(S\cup\{x\}))\leqslant \HA(P\|Q)\\
		&\Leftrightarrow T[M(S\cup\{x\}), M(S)]\geqslant T[Q,P]
		&& (\text{\Cref{lem:HSfromf}})\\
		&\Leftrightarrow T[M(S), M(S\cup\{x\})]\geqslant T[P,Q]
		\Leftrightarrow\text{condition 3}
		&& (*)\\
		&\Leftrightarrow\HA(M(S\cup\{x\}),M(S))\leqslant \HA(Q\|P)
		\Leftrightarrow\text{condition 2}
	\end{align*}
	Here $(*)$ uses the fact that $T[P,Q]$ and $T[Q,P]$ are inverse functions of each other.
\end{proof}

The next lemma plays the central role in both parts of \Cref{thm:amplification_by_sampling}.
Suppose we have probability distributions $P_1,\ldots, P_n$ and $Q_1,\ldots, Q_n$, all on the same domain and let $\bar{P} = \sum_{i=1}^n p_i P_i$ and $\bar{Q}=\sum_{i=1}^n p_i Q_i$ be the corresponding mixture distributions with the same coefficients $p_1,\ldots, p_n$ where $\sum_{i=1}^n p_i =1$ and all $p_i\geqslant 0$. Then we have
\begin{lemma}\label{lem:comparison}
	If $T[P_i,Q_i]\geqslant T[P,Q]$ for all $i\in[n]$, then for any $\gamma\in[0,1]$, we have
	\[T\big[\bar{P}, (1-\gamma) \bar{P}+ \gamma \bar{Q}\big]\geqslant T[P,(1-\gamma) P+ \gamma Q].\]
\end{lemma}
\begin{proof}[Proof of \Cref{lem:comparison}]
	Let $f=T[P,Q]$. First we claim that 
	\begin{equation}\label{eq:2156}
		T[P,(1-\gamma) P+ \gamma Q](x) = (1-\gamma)(1-x)+ \gamma f(x).
	\end{equation}
	To see this, consider any testing rule $\phi$ such that $\E_P[\phi]= x$. We need to show that 
	$$\inf_{\phi:\E_P[\phi]= x}\E_{(1-\gamma) P+ \gamma Q}[1-\phi]= (1-\gamma)(1-x)+ \gamma f(x).$$
	By definition of $f=T[P,Q]$, we have
	$\inf_{\phi:\E_P[\phi]= x}\E_Q[1-\phi]= f(x)$. Therefore,
	\begin{align*}
		\E_{(1-\gamma) P+ \gamma Q}[1-\phi]
		&=(1-\gamma) \E_P[1-\phi]+ \gamma \E_Q[1-\phi]\\
		&= (1-\gamma)(1-x)+ \gamma \E_Q[1-\phi]\\
		\inf_{\phi:\E_P[\phi]= x}\E_{(1-\gamma) P+ \gamma Q}[1-\phi]
		&=(1-\gamma)(1-x)+ \gamma \inf_{\phi:\E_P[\phi]= x}\E_Q[1-\phi]\\
		&=(1-\gamma)(1-x)+ \gamma f(x).
	\end{align*}
	This verifies \eqref{eq:2156}. Next we proceed to the proof of the lemma.
	Similarly, it suffices to consider arbitrary testing rules $\phi$ with $\E_{\bar{P}} [\phi]\leqslant x$ and show
	\begin{equation}\label{eq:fe}
		\E_{(1-\gamma) \bar{P}+ \gamma \bar{Q}} [1-\phi]\geqslant 
		T[P,(1-\gamma) P+ \gamma Q](x) = (1-\gamma)(1-x)+ \gamma f(x).
	\end{equation}
	Expanding the convex combination, we have
	\begin{align*}
		\E_{(1-\gamma) \bar{P}+ \gamma \bar{Q}} [1-\phi]
		&=(1-\gamma) \E_{\bar{P}}[1-\phi]
		+ \gamma  \sum p_i \E_{Q_i}[1-\phi]\\
		&\geqslant (1-\gamma) (1-x)+ \gamma  \sum p_i \E_{Q_i}[1-\phi]
	\end{align*}
	Comparing to \eqref{eq:fe}, it suffices to show
	\[\sum p_i \E_{Q_i}[1-\phi]\geqslant f( x)\]
	We know that $T[P_i,Q_i]\geqslant f$. Hence $\E_{Q_i}[1-\phi]\geqslant f(\E_{P_i}[\phi])$. By convexity of $f$,
	\begin{align*}
		\sum p_i \E_{Q_i}[1-\phi]
		\geqslant \sum p_i f(\E_{P_i}[\phi])\geqslant f\left(\sum p_i \E_{P_i}[\phi]\right)=f( \E_{\bar{P}}[\phi])\geqslant f(x).
	\end{align*}
	The last inequality follows from the monotonicity of trade-off functions. Hence \eqref{eq:fe} is verified and the proof is complete.
\end{proof}

\begin{proof}[Proof of \Cref{thm:amplification_by_sampling} (1)]

By \Cref{lem:primal}, it suffices to prove 
\begin{equation}\label{eq:goalpois}
	T\big[M(S_\pois),M(S'_\pois)\big]\geqslant T[P,(1-\gamma) P+ \gamma Q]
\end{equation}
where $S=\{x_1,\ldots, x_{n-1}\}$ and $S'=S\cup\{x_n\}$.
The outcome of repeated coin flips can be labeled as $\vec{b}\in\{0,1\}^{n-1}$. We use $S_{\vec{b}}$ to denote the corresponding subset of $S$ and $S_{\vec{b}0}'$ or $S_{\vec{b}1}'$ for that of $S'$, depending on whether $x_n$ is included. Note that $S_{\vec{b}0}' = S_{\vec{b}}$. Furthermore, let $p_{\vec{b}}$ be the probability of the outcome $\vec{b}$ (recall that each coin is a Bernoulli $\gamma$ random variable). In fact, $p_{\vec{b}} = \prod_{i=1}^{n-1}\gamma^{b_i}(1-\gamma)^{1-b_i}$ but we will not use it.

Both $M(S_\pois)$ and $M(S'_\pois)$ are mixtures. We have the following decompositions
\begin{align*}
	M(S_\pois) &= \sum_{\vec{b}\in\{0,1\}^{n-1}} p_{\vec{b}}M(S_{\vec{b}})\\
	M(S'_\pois) &= \sum_{\vec{b}\in\{0,1\}^{n-1}} (1-\gamma) p_{\vec{b}}M(S'_{\vec{b}0})+\gamma p_{\vec{b}}M(S'_{\vec{b}1})\\
	&=(1-\gamma) M(S_\pois)+\gamma \sum_{\vec{b}\in\{0,1\}^{n-1}} p_{\vec{b}}M(S'_{\vec{b}1})
\end{align*}
Now we are ready to use \Cref{lem:comparison}, with the family of $P_i$ being $M(S_{\vec{b}})$ and the family of $Q_i$ being $M(S_{\vec{b}1}')$. By the calculation above, $M(S_\pois)$ will be the $\bar{P}$ in \Cref{lem:comparison} and $M(S'_\pois)$ is exactly the $(1-\gamma) \bar{P}+ \gamma \bar{Q}$ in \Cref{lem:comparison}.
We still need to verify the condition in \Cref{lem:comparison}: since $(P,Q)$ dominates $M$ for add neighbors, for each $\vec{b}\in\{0,1\}^{n-1}$ we have
\[T\big[M(S_{\vec{b}}), M(S'_{\vec{b}1})\big]\geqslant T[P,Q].\]
Therefore, \Cref{lem:comparison} gives us \eqref{eq:goalpois}, which is exactly what we want.
\end{proof}
\begin{proof}[Proof of \Cref{thm:amplification_by_sampling} (2)]
By \Cref{lem:primal}, it suffices to prove 
\begin{equation}\label{eq:goal}
	T\big[M(S_\unif),M(S'_\unif)\big]\geqslant T[P,(1-\gamma) P+ \gamma Q]
\end{equation}
where $S=\{x_1,\ldots, x_{n-1}\}$ and $S'=S\cup\{x_n\}$.


Below we use notations such as $S_I, S_{J\cup\{n\}}$ to denote the (obvious) subsets of $S$ where $I,J\subseteq [n-1]$ and $|I|=m, |J|=m-1$ consistently. Both $M(S_\unif)$ and $M(S'_\unif)$ are mixtures. We have the following decompositions, where the latter is further decomposed into two parts depending on whether $n$ is selected.
\begin{align}
	M(S_\unif) &= \sum_{|I|=m} \frac{1}{\binom{n-1}{m}}M(S_I)\label{eq:3}\\
	M(S'_\unif) &= \sum_{|J|=m-1} \frac{1}{\binom{n}{m}}M(S_{J\cup\{n\}})+
	\sum_{|I|=m} \frac{1}{\binom{n}{m}}M(S_{I})\label{eq:4}
\end{align}
It's not in a ready shape to use \Cref{lem:comparison}. We need to further break the summands by carefully creating copies of the components. Let
\begin{align*}
\cI&{\,:=\,}\{I^{(k)}:I\subseteq [n-1], |I|=m,1\leqslant k \leqslant m\}\\
\cJ&{\,:=\,}\{J^{(l)}\cup\{n\}:J\subseteq[n-1], |J|=m-1, 1\leqslant l \leqslant n-m\}.
\end{align*}
That is, we create $m$ copies of each subset of $[n-1]$ of cardinality $m$ and collect as $\cI$; create $n-m$ copies of each subset of $[n]$ of cardinality $m$ that includes $n$ and collect as $\cJ$.
Now we claim two things
\begin{enumerate}
	\item  There is a bijection $F:\cI\to\cJ$ such that $F(I)$ differs from $I$ by one element for all $I\in \cI$.
\item Let $p_I = \frac{1}{\binom{n-1}{m}m}$ for all $I\in\cI$ and $\gamma = \frac{m}{n}$. Then
\begin{align}
	M(S_\unif) &= \sum_{I\in\cI} p_I M(S_I)\label{eq:5}\\
	M(S'_\unif) &= \gamma \sum_{I\in\cI} p_I M(S_{F(I)})+
	(1-\gamma)\sum_{I\in\cI} p_I M(S_{I})\label{eq:6}
\end{align}
\end{enumerate}
Now we are ready to use \Cref{lem:comparison}: the collection $\{P_i\}$ is $\{M(S_{I})\}$ and $\{Q_i\}$ is $\{M(S_{F(I)})$. The conclusion is
 exactly \eqref{eq:goal}.

Next we turn our attention to the proofs of the two claims.

For claim 1, let's first construct a $(n-m)$-to-one surjective map $\tilde{F}:\cI\to \tilde{\cJ}$ where $\tilde{\cJ}=\{J\cup\{n\}:J\subseteq[n-1], |J|=m-1\}$.
$\tilde{F}(I^{(k)})$ is obtained by replacing the $k$-th element in $I$ by $n$. We see that $\tilde{F}(I^{(k)})$ is indeed in $\tilde{\cJ}$. For any $J\cup\{n\}\in\tilde{\cJ}$, it is hit by $\tilde{F}$ exactly $n-m$ times since the $n$ could have been any
of the $(n-1)-(m-1) = n-m$ indices not already in $J$.

Since $\cJ$ contains $n-m$ copies of $\tilde{\cJ}$, the $(n-m)$-to-one mapping $\tilde{F}:\cI\to \tilde{\cJ}$ can be ``redirected'' in an obvious way and become a bijection between $\cI$ and $\cJ$. By construction, $F(I)$ and $I$ differ in exactly one element.

For claim 2, \eqref{eq:5} is easier and we only show \eqref{eq:6}. Since we have created copies, by splitting the summands of \eqref{eq:4}, we have
\begin{align*}
	M(S'_\unif) &= \sum_{|J|=m-1} \frac{1}{\binom{n}{m}}M(S_{J\cup\{n\}})+
	\sum_{|I|=m} \frac{1}{\binom{n}{m}}M(S_{I})\\
	&= \sum_{J\in\cJ} \frac{1}{\binom{n}{m}\cdot(n-m)}M(S_{J})+
	\sum_{I\in\cI} \frac{1}{\binom{n}{m}m}M(S_{I})\\
	&= \sum_{I\in\cI} \frac{1}{\binom{n}{m}\cdot(n-m)}M(S_{F(I)})+
	\sum_{I\in\cI} \frac{1}{\binom{n}{m}m}M(S_{I})\\
\end{align*}
Comparing to \eqref{eq:6}, it suffices to show
\[\gamma p_I = \frac{1}{\binom{n}{m}\cdot(n-m)}\text{ and }(1-\gamma)p_I = \frac{1}{\binom{n}{m}m}\]
which are elementary combination identities.
\end{proof}

	\begin{remark}[Exact optimality of the bounds]
	If $(P,Q)$ is a \emph{tightly dominating pair} for $\cM$, for both ``Removal''-neighboring relation or ``Add''-neighboring relation, then under some mild regularity conditions on $\cM$ and the space of the input datasets,	
	Theorem~\ref{thm:amplification_by_sampling} can be strengthened to show that that $((1-\gamma)Q + \gamma P, Q )$ and $(P, (1-\gamma)P + \gamma Q)$ are \emph{tight dominating pairs} for the ``Removal''-neighboring relation and  ``Add''-neighboring relation respectively --- i.e., the dominating pair is realized by some concrete datasets.  
	For example, consider $\cM$ to be Gaussian mechanism or Laplace mechanism that releases the total number of $1$s in a dataset.  Then two neighboring datasets $D = [0,...,0,0,1]\in \R^{n+1}$, $D' = [0,0,...,0]\in \R^{n}$ for ``removal'' and $D = [0,...,0]\in \R^{n}$, $D' = [0,0,...,0,1]\in \R^{n+1}$ for ``addition'' attains the upper bound for all $\alpha>0$ in each category.
\end{remark}

\begin{remark}[Renyi DP and Optimal Moments Accountant for subsampled mechanisms]\label{rem: rdp}
	Renyi-DP and moments accountant are closely related concepts that are often considered identical. However, our results suggest that there is a distinction.  The above pair of $P,Q$ we constructed are not necessarily attaining the Renyi-DP bounds (see a concrete example from~\cite{zhu2019poission}, but as moments accountant focuses only on computing $(\epsilon,\delta)$-DP, it suffices use the Renyi-divergence functions $R_{\alpha}(P\|Q)$.    Specifically, this closes the constant gap between the moments accountant for subsampled mechanisms and Poisson sampled mechanisms.
\end{remark}

\subsection{Other schemes in privacy-amplification for dominating pairs}
	Theorem~\ref{thm:amplification_by_sampling} provides new results and a novel composition algorithm for the popular Poisson sampling under ``add/remove'' neighboring relations by treating ``add'' and ``remove'' separately. It also shows that  the same result hold in a not-so-typical but practically relevant scheme of the \emph{random-subset} sampled mechanism under the ``add / remove'' neighboring relation for the case when the base mechanism's privacy is defined by the ``replace''  neighboring relation. 
	
	What we left unspecified is whether there is a clean dominating pair under the two alternative schemes: (a) Poisson-sampling + ``Add/remove''  \emph{without} treating ``Add'' or ``Remove'' separately;  (b) Subset sampling + ``Replace'' neighboring relation for both $\cM\circ\mathrm{Sample}_\gamma$ and $\cM$.
	
	It turns out that  while we can construct these dominating pairs of $\cM\circ\mathrm{Sample}_\gamma$ explicitly based on the dominating pair  $(P,Q)$ of $\cM$, but the expression is not simple.  We present these results in this section.
	
	To avoid any confusion, for all practical purposes, the result in Theorem~\ref{thm:amplification_by_sampling} suffices because we can always compose ``Add'' or ``Remove'' separately and only take the pointwise maximum in the end, while only incurring twice as much computation, but the results in this section are interesting from a purely scientific perspective and they are included for the completeness in our understanding of the problem.
	
	\begin{proposition}\label{prop:sampling_alpha_geq1}
		If $(P,Q)$ is a dominating pair of $\cM$ under ``Add/remove'' Relation, then 
		\begin{align*}
			\delta_{\cM\circ S_{\textbf{Poisson}}}(\alpha) \leq \begin{cases}
				H_{\alpha}((1-\gamma)Q + \gamma P, Q )  & \text{ for }\alpha \geq 1;\\
				H_{\alpha}(P, (1-\gamma)P + \gamma Q )  & \text{ for }0<\alpha < 1.
			\end{cases}
		\end{align*}
	under the ``Add/Remove'' relation.
Similarly, if $(P,Q)$ is a dominating pair of $\cM$ under ``Replace'' relation for dataset of size $\gamma n$, then
		\begin{align*}
	\delta_{\cM\circ S_{\textbf{Subset}}}(\alpha) \leq \begin{cases}
		H_{\alpha}((1-\gamma)Q + \gamma P, Q )  & \text{ for }\alpha \geq 1;\\
		H_{\alpha}(P, (1-\gamma)P + \gamma Q )  & \text{ for }0<\alpha < 1.
	\end{cases}
\end{align*}
under ``Replace'' relation for dataset of size $n$.
	\end{proposition}
	
We plot $H_{\alpha}((1-\gamma)Q + \gamma P, Q )$ and $H_{\alpha}(P, (1-\gamma)P + \gamma Q ) $ for $	\delta_{\cM\circ S_{\textbf{Poisson}}}(\alpha)$ in Figure~\ref{figure: hs}.
\begin{figure*}[t]
	\centering	
	\subfigure[HS divergence]{
		\includegraphics[width=0.50\textwidth]{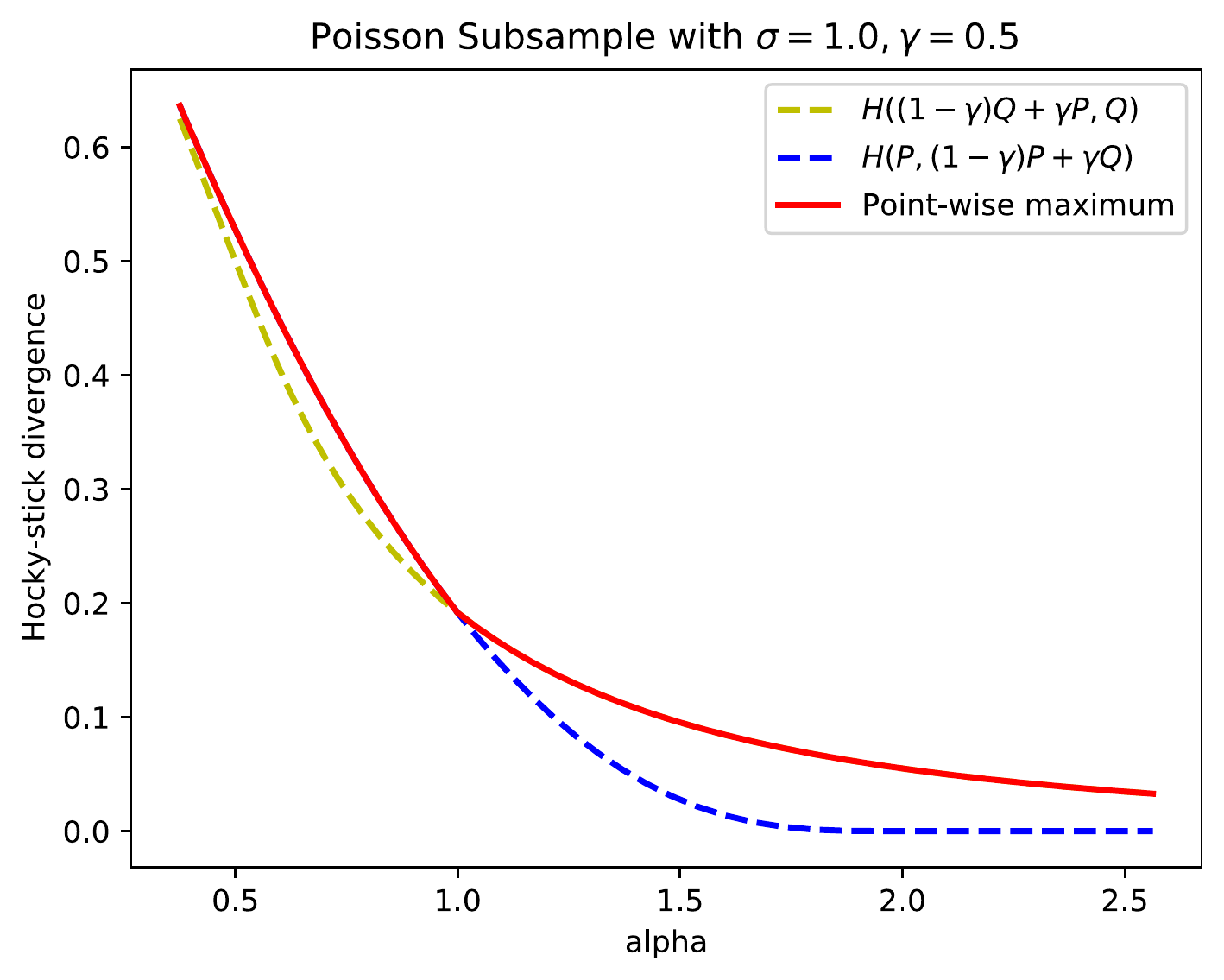}\label{figure: hs}}
	\caption{ We consider Poisson subsample Gaussian mechanism with $\sigma=1.0$ and $\gamma =0.5$. The red line denotes the pointwise maximum over two HS curves. Figure shows that $H_{\alpha}((1-\gamma)Q + \gamma P, Q )$ dominates the region $\alpha\geq 1$ and   $H_{\alpha}(P, (1-\gamma)P + \gamma Q )$ dominates the region $\alpha<1$.} 
	\vspace{-3mm}
	\label{fig: exp_figure}
\end{figure*}

	The proof of the above result requires the use of the following general result that establishes the relationship between pairs that dominate only one half of the range for $\alpha$ and those that dominate the other half.
	\begin{lemma}[Properties for  ``symmetric neighbors'']\label{lem:symmetric_neighbors}
		Let $\cM$ be a mechanism and $\simeq$ be a \emph{symmetric} neighboring relationships , i.e., $D\simeq D' \Leftrightarrow D\simeq D'$.  
		Then \begin{enumerate}
			\item  If  $(P,Q)$ is a dominating pair of $\cM$, then $(Q,P)$ is also a dominating pair of $\cM$,
			\item If $\sup_{D\simeq D'}H_{\alpha}(\cM(D)\|\cM(D')) \leq \min\{H_{\alpha}(P\|Q), H_{\alpha}(Q\|P)\} \text{ for all }\alpha \geq  1 $, then $(P,Q)$ and $(Q,P)$ are both dominating pairs of $\cM$.
			\item The following two statements are equivalent.
			\begin{enumerate}[label=(\alph*)]
				\item 		$
				\sup_{D\simeq D'}H_{\alpha}(\cM(D)\|\cM(D')) \leq H_{\alpha}(P\|Q) \text{ for all }\alpha \geq  1.
				$
				\item  		$
				\sup_{D\simeq D'}H_{\alpha}(\cM(D)\|\cM(D')) \leq H_{\alpha}(Q\|P) \text{ for all }0 <\alpha \leq  1.
				$
			\end{enumerate}
		\end{enumerate}
	\end{lemma}
\begin{proof}
	First notice that the first and second statements are both implied by the third.  For the first statement, notice that by the definition of the dominating pair, the upper bound applies for all $\alpha >0$. Thus  by applying $(a)\Rightarrow(b)$ and $(b)\Rightarrow(a)$ for $(P,Q)$, we get that $(Q,P)$ is also a dominating pair. For the second statement, we apply $(a)\Rightarrow(b)$  for both $(P,Q)$ and $(Q,P)$, then the result extends the bound to the full range.
	
	It remains to prove the third statement. For any pair of neighboring $D,D'$ and $0<\alpha \leq 1$, by Lemma~\ref{lem:HSpq}. 
	\begin{align*}
	H_{\alpha}(\cM(D)\|\cM(D')) &= \alpha H_{\alpha-1}(\cM(D')\|\cM(D)) - \alpha + 1\\
	&\leq \alpha H_{\alpha^{-1}}(P\|Q) - \alpha + 1 \\
	&= H_{\alpha}(Q\|P),
	\end{align*}
where the inequality uses the fact that $\alpha^{-1}\geq 1$, $\simeq$ is symmetric, and that $(P,Q)$ dominates for order $\geq 1$.  The converse follows the same argument but starts with $\alpha >1$.
	\end{proof}

\begin{proof}[Proof of Proposition~\ref{prop:sampling_alpha_geq1}]
		
		We first prove the case for $\alpha \geq 1$. By Theorem 8 of \citep{balle2018couplings} (Standard amplification by sampling bound in DP), we know that for any  $D \simeq_{\textbf{add/remove}} D'$ and all $\epsilon \geq 0$ (thus $\alpha \geq 1$!)
		\begin{align*}
			&H_{e^\epsilon}(\cM\circ S_{\textbf{Poisson}}^\gamma(D) \| \cM\circ S_{\textbf{Poisson}}^\gamma(D')) \leq \gamma \sup_{\tilde{D}\simeq_{\text{add/remove}} \tilde{D}'} H_{1 + \frac{e^\epsilon-1}{\gamma}}(\cM(\tilde{D})\|\cM(\tilde{D}'))  \\
			\leq&  \gamma H_{1 + \frac{e^\epsilon-1}{\gamma}}(P\|Q) = \gamma \int \left(p(x)-(1+\tfrac{e^\epsilon-1}{\gamma})q(x)\right)_+\diff x\\
			=& \int \left((1-\gamma)q(x)+\gamma p(x) - e^\epsilon q(x)\right)_+\diff x= H_{e^\epsilon}((1-\gamma)Q + \gamma P \|  Q).
		\end{align*}
%
%
		where the inequality in the second line is due to that $(P,Q)$ is a dominating pair of $\cM$.
		
		The same proof works line by line for the random subset sampling when we use Theorem 9  of \citep{balle2018couplings} instead and adopt the $D \simeq_{\textbf{replace}} D'$ neighboring relationship.
		
		Next we prove the statement for $0<\alpha <1$.  Check that we can apply Lemma~\ref{lem:symmetric_neighbors} because both $\simeq_{\textbf{replace}}$ and $\simeq_{\textbf{add/remove}}$ are symmetric.  By the third statement of Lemma~\ref{lem:symmetric_neighbors}, 
		$(Q,(1-\gamma)Q + \gamma P )$ dominates the subsampled mechanism for $0<\alpha < 1$.  Also by the first statement of Lemma~\ref{lem:symmetric_neighbors} we know that $(Q,P)$ is also a dominating pair for $\cM$, thus by repeating the same argument, we get that   $(P,(1-\gamma)P + \gamma Q)$ also dominates  the subsampled mechanism for $0<\alpha < 1$, which completes the proof.
\end{proof}

The above discussion characterizes the tight\footnote{They are tight for the reason we described in the Remark on the ``Exact optimality of the bounds'' above.} upper bound of the privacy profile of subsampled mechanisms using two (ordered) pairs of distributions, rather than just one pair. This is insufficient for us to apply the composition theorem  because the region between $\alpha >1$ and $\alpha<1$ in some sense ``mixes with each other'' during composition, as we have clearly seen from the proof of Theorem~\ref{thm:composition}.  What we do know is that neither $((1-\gamma)Q + \gamma P, Q )$ nor $(Q,(1-\gamma)Q + \gamma P)$ is a dominating pair for the sampled mechanism under ``Add/Remove'' or ``Replace'' neighboring relation. This is the reason why we proposed the more elegant approach for handling ``add'' and ``remove'' separately in the first place.  

For completeness, and to also handle the case when we want the ``replace one'' neighboring relation for $\cM\circ S_{\textbf{Subset}}$, we state the following result which constructs an explicit but not-so-clean  dominating pair. 
\begin{corollary}[Dominating pair of sampled mechanisms under symmetric neighbor relations] Let $(P,Q)$ be a dominating pair of $\cM$ and $\simeq$ is ``Add/Remove'' (or ``Replace'').  Then $(\tilde{P},\tilde{Q})$ is  a dominating pair of $\cM\circ S_{\textbf{Poisson}}$ (or $\cM\circ S_{\textbf{Subset}}$) where 
	$
	\tilde{Q} = \textrm{Uniform}[0,1]
	$
	and 
	$\tilde{P}$ has a CDF of 
	$$
	F_{\tilde{P}}(x) = \begin{cases}
		0, &\text{ if }x<0;\\
		1 + \Big[\max\{ H_{\alpha}((1-\gamma)Q + \gamma P, Q ) ,  H_{\alpha}(P, (1-\gamma)P + \gamma Q ) \}\Big]^* (-x), &\text{ if } x \in [0,1),\\
		1,& \text{ if }x \geq 1;  
	\end{cases}
	$$	
	where $\big[\cdot\big]^*$ denotes the \emph{Fenchel conjugate} of a function.
\end{corollary}
The proof, which we omit, is a direct application of  Proposition~\ref{prop:sampling_alpha_geq1} with the generic approach of Proposition~\ref{prop:whatisHS} for constructing the dominating pair.

\section{The characteristic function of basic mechanisms}

\subsection{$\phi$-function of basic mechanisms}

We  now derive $\phi$-function for three basic mechanisms: randomized response, Laplace and Gaussian mechanism. The results and their dominating distributions are summarized in Table~\ref{tab: summary}. 

Let $f$ be a predicate, i.e., $f: \cD^* \to \{0,1\}$. The Randomized Response mechanism for $f$ is defined as 
\[
\cM(D)=
\begin{cases}
f(D) & \text{with probabilty }p\\
1-f(D)& \text{with probability} 1-p
\end{cases}
\]

\begin{lemma}[Randomized response]
	The $\phi$ function of  Randomized Response mechanism $\cM$ with the parameter $p$  satisfies 
	$
	\phi_{\cM}(\alpha) = \phi’_{\cM}(\alpha) = p e^{\alpha i \log(\frac{p}{1-p})} + (1-p)e^{\alpha i \log(\frac{1-p}{p}).}
	$
\end{lemma}
\begin{proof}
	First of all, the dominating pair will be:
	\[P: \pr_P[0]=p; Q:\pr_Q[1] = p\]
	Then, follow the definition of $\phi$-function, we have
	\[\phi_\cM(\alpha) = \mathbb{E}_{x\sim P} e^{i\alpha \log(p(x)/q(x))} = pe^{\alpha i \log(\frac{p}{1-p})} +(1-p)e^{\alpha i \log(\frac{1-p}{p})}\]
\end{proof}
For Laplace and Gaussian mechanisms, we assume that $f: \cD^* \to \cR$ is a function of sensitivity 1.

\begin{lemma}[Laplace Mechanism]\label{prop: laplace}
	Let Laplace Mechanism for $f$  is defined as 
	$
	\cM(D) = f(D) + \mathrm{Lap}(\lambda)
	$
	where $\mathrm{Lap}(\lambda)$ is Laplace dstribution with scale $\lambda$, i.e., its density function is $\frac{1}{2\lambda}\exp(-|x|/\lambda)$.
	For any $\alpha \in \cR$ and $\lambda >0$, we have
	\[\phi_{\cM}(\alpha) = \phi'_{\cM}(\alpha) =\frac{1}{2}\bigg(e^{\frac{i\alpha}{\lambda}} + e^{\frac{-i\alpha -1}{\lambda}}+ \frac{1}{2\alpha i +1} (e^{\frac{i \alpha }{\lambda}}-e^{\frac{-i \alpha  -1}{\lambda}})\bigg).\]
\end{lemma}

\begin{proof}
	 we consider the dominating distribution $P(o) = \frac{1}{2\lambda} \exp(-|o|/\lambda)$ and $Q(o) = \frac{1}{2\lambda} \exp(-|o-1|/\lambda)$. 
	 We can show that the dominating pair for $\alpha\geq 1$ also dominates $0<\alpha<1$ using the third statement of Lemma~\ref{lem:symmetric_neighbors} and the symmetry of Laplace mechanism.
	To calcuate the characteristic function $\phi_{\cM}$, we define the privacy loss RV $L_{P, Q}$ as follows
	\[L_{P,Q}(o)=\log\bigg(\frac{p(o)}{q(o)}\bigg)=
	\begin{cases}
		\frac{1}{\lambda}, & o\leq 0\\
		\frac{1-2o}{\lambda}, & 0<o\leq 1\\
		\frac{-1}{\lambda}, & o>1
	\end{cases}
	\]
	
	The characteristic function $\phi_{\cM}(\alpha)$ is calculated as follows 
	\begin{align*}
		&\phi_{\cM}(\alpha)=\mathbb{E}_P[e^{i\alpha L_{P,Q}}]\\
		& = \int_{-\infty}^0 e^{\frac{i\alpha}{\lambda}}\frac{1}{2\lambda}e^{\frac{o}{\lambda}}do + \int_{0}^1 e^{\frac{i\alpha(1-2o)}{\lambda}}\frac{1}{2\lambda}e^{\frac{-o}{\lambda}}do + \int_{1}^\infty e^{\frac{-i\alpha}{\lambda}}\frac{1}{\lambda}e^{\frac{-o}{\lambda}}do\\
		&=\frac{1}{2\lambda}\cdot \bigg(e^{\frac{i\alpha}{\lambda}}\lambda +\frac{\lambda}{2it+1}(e^{\frac{i\alpha}{\lambda}}-e^{\frac{-i\alpha-1}{\lambda}})+\lambda e^{-1/\lambda}e^{\frac{-i\alpha}{\lambda}}\bigg)\\
		& = \frac{1}{2}\bigg( e^{\frac{i\alpha}{\lambda}}+\frac{1}{2i\alpha+1}(e^{\frac{i\alpha}{\lambda}}-e^{\frac{-i\alpha-1}{\lambda}})+ e^{\frac{i\alpha+1}{-\lambda}}\bigg)
	\end{align*}
	Similarly, for $\phi'_{\cM}$, we define privacy loss RV $L_{Q, P}$ as follows 
	\[L_{Q,P}(o)=\log\bigg(\frac{q(o)}{p(o)}\bigg)=
	\begin{cases}
		\frac{-1}{\lambda}, & o\leq 0\\
		\frac{2o-1}{\lambda}, & 0<o\leq 1\\
		\frac{1}{\lambda}, & o>1
	\end{cases}
	\]
	Therefore, we have 
	\begin{align*}
		&\phi'_{\cM}(\alpha) = \int_{-\infty}^0 e^{\frac{-|o-1|}{\lambda}}\cdot e^{\frac{-\alpha i}{\lambda}}do + \frac{1}{2\lambda}\int_{0}^1 e^{\frac{-|o-1|}{\lambda}}e^{\frac{2o -1}{\lambda}i\alpha}do + \int_{1}^\infty e^{\frac{-|o-1|}{\lambda}}e^{\frac{i\alpha}{\lambda}}do\\
	& =  \frac{1}{2}\bigg( e^{\frac{i\alpha}{\lambda}}+\frac{1}{2i\alpha+1}(e^{\frac{i\alpha}{\lambda}}-e^{\frac{-i\alpha-1}{\lambda}})+ e^{\frac{i\alpha+1}{-\lambda}}\bigg)
\end{align*}
\end{proof}
\begin{lemma}[Gaussian Mechanism]\label{prop: gaussian}
	Let Gaussian mechanism is defined as 
	$
	\cM(D) = f(D) + \cN(0,\sigma^2)
	$.	For any $\alpha \in \cR$ and $\sigma>0$,  we have
	$\phi_{\cM}(\alpha) = \phi’_{\cM}(\alpha) = e^{\frac{-1}{2\sigma^2}(\alpha^2 - i\alpha)}.$
\end{lemma}

\begin{proof}
	For Gaussian mechanism, no matter what the dimensionality of the output sapce is, the dominating distributions will always be 1D, and that extends to subsampled-gaussian as well.
	In the proof, we consider the worst-case pair $p(o) = \frac{1}{\sqrt{2\pi \sigma^2}}e^{-\frac{(o-1)^2}{2\sigma^2}}, q(o) = \frac{1}{\sqrt{2\pi \sigma^2}}e^{-\frac{o^2}{2\sigma^2}}$.

	Follow the definion of privacy loss RV, we have $L_{P, Q}(o)=\frac{2o-1}{2\sigma^2}$ and $L_{Q, P}(o)=\frac{1-2o}{2\sigma^2}$.  Then we have
	\[\phi_\cM(\alpha)[e^{i\alpha \log(p/q)}] = \int_{-\infty}^\infty \frac{1}{\sqrt{2\pi \sigma^2}}e^{\frac{(o-1)^2}{-2\sigma^2}}e^{\frac{(2o-1)i\alpha}{2\sigma^2}}do = e^{\frac{\alpha^2 - i\alpha}{-2\sigma^2}}\]
	Similarly,
		\[\phi'_\cM(\alpha)[e^{i\alpha \log(p/q)}] = \int_{-\infty}^\infty \frac{1}{\sqrt{2\pi \sigma^2}}e^{\frac{(o)^2}{-2\sigma^2}}e^{\frac{(1-2o)i\alpha}{2\sigma^2}}do = e^{\frac{\alpha^2 - i\alpha}{-2\sigma^2}}\]
	
\end{proof}

Besides basic mechanisms, all mechanisms with discrete outputs admit an analytical $\phi$ function by definition.   
\begin{definition}[$\phi$-function of mechanisms with discrete outputs] \label{def: phi_discrete}
	Let $p,q$ be the probability mass function induced by $\cM$,
	\[\phi_{\cM}(\alpha) = \sum_{o\in \cO} e^{\log p(o) + i\alpha (\log p(o) - \log q(o))} \text{   and   }
	\phi'_{\cM}(\alpha) = \sum_{o\in \cO} e^{\log q(o) + i\alpha (-\log p(o) + \log q(o))}.\]
\end{definition}
This function can be represented by two vectors that lists the probability masses at $o$ from $p$ and $q$. When evaluating $\log \phi_{\cM}(\alpha)$ at a given $\alpha$, we could use the log-sum-exp trick to improve the numerical stability.  Overall the space and time in representing these functions are linear in the size of the output space.

Provided that the worst-case pair of distributions are known, this procedure allows us to compose over exponential mechanisms \citep{mcsherry2007mechanism}, Report-noisy-max, as well as other complex mechanisms that arise out of post-processing of continuous output mechanisms, e.g., NoisyScreening \citep{zhu2020private} that had been used as a practical alternative to sparse vector techniques.  

\subsection{Handling probability mass at $\infty$}\label{sec:mass_at_inf}
One of the motivations of the work is for us to handle the situations where there is a non-zero probability mass where the privacy loss r.v. is at infinity.  This naturally happens in propose-test-rease-style algorithm \citep{dwork2009differential} where we first construct a differentially private upper bound of the local sensitivity (or other data-dependent quantities) then calibrate noise according to the local sensitivity.  The issue is that there is always a non-zero probability where the upper bound is not valid. Standard $(\epsilon,\delta)$-DP handles this case at ease, but modern techniques such as RDP struggles as such a mechanism does not satisfy RDP for any $\alpha>0$.

In this case, Lemma~\ref{lem:equivalent_defs} can be more explicitly rewritten into.
\begin{lemma}
	Let $D,D'$ be the worst-case pair of datasets for $\cM$, and $P = \cM(D), Q=\cM(D')$ then
	$$
	\delta_\cM(\epsilon) = \max\{\delta_{P}(\epsilon), \delta_{Q}(\epsilon)\}
	$$
	where 
	\begin{align*}
		\delta_{P}(\epsilon) = \pr[L_{P, Q}= \infty] + \bigg( \pr[\epsilon< L_{P,Q} < \infty] -e^{\epsilon}\pr[L_{Q,P} < -\epsilon]\bigg)\\
		\delta_{Q}(\epsilon) = \pr[L_{Q,P}= \infty] + \bigg( \pr[\epsilon< L_{Q,P} < \infty] -e^{\epsilon}\pr[L_{P,Q} < -\epsilon]\bigg)
	\end{align*}
\end{lemma}


\begin{proposition}[Composition]
	
	Let $L_{P, Q'}^{(1)}, ..., L_{P, Q'}^{(k)}$ be the privacy loss R.V. of mechanism $\cM_1, ..., \cM_k$ ($L_{Q', P}^{(i)}$ is defined analogously). Then,  the  $\delta_P(\epsilon)$ and $\delta_Q(\epsilon)$ of the composed mechansims are as follow:
	\begin{align*}
		\delta_{P}(\epsilon) = 1 - \prod_{i=1}^k(1-\pr[L_{P,Q'}^{(i)}= \infty] )+ \bigg( \pr[\epsilon< \tilde{L}_{P,Q'} < \infty] -e^{\epsilon}\pr[\tilde{L}_{Q',P} < -\epsilon]\bigg)\\
		\delta_{Q}(\epsilon) = 1 - \prod_{i=1}^k(1-\pr[L_{Q',P}^{(i)}= \infty] ) + \bigg( \pr[\epsilon< \tilde{L}_{Q',P} < \infty] -e^{\epsilon}\pr[\tilde{L}_{P,Q'} < -\epsilon]\bigg)
	\end{align*}
where $\tilde{L}_{P,Q'}$ is defined as the privacy loss R.V. of the composed mechanisms that excludes $\infty$. 
\end{proposition}

The above essentially says that we can handle the cases with $\infty$ separately, and compose the characteristic function of the sub-probability measure that excludes $\infty$.

\section{$\phi$-function with discretization and experimental details}\label{sec: experiment}
There are cases when the closed-form $\phi$-functions do not exist.  For example, in the subsample mechanisms, the privacy loss distribution is complicated and continuous, suggesting that we cannot derive an exact closed-form expression naively.  In this section, we first provide a discretization-based solution and analyze its error bound. Later, we develop an efficient approximation method, ``Double quadrature''.

\begin{algorithm}[H]
	\caption{$\phi$-function approximation}
	\label{alg: appro}
	{\bf Input:}
	\ The output interval $[-S, S]$, pdf $p(\cdot), q(\cdot)$ and $N$.
	\begin{algorithmic}[1]
		\STATE Set $N$ equidistant points: $o_j = -S + j\triangle_o$, where $\triangle_o = 2S/N, j = 0,...,N-1$.
		\STATE $ L_{P, Q}^-, L_{Q, P}^-\gets$ the lower approximations of $L_{P, Q}, L_{Q, P}$ using Definition~\ref{def: appro}.
		\STATE Construct $\phi_{\cM_P}^-(\alpha)$: $\phi_{\cM_P}^-(\alpha) \gets \sum_{o_j} e^{\log \tilde{p}(o_j) + i\alpha L_{P, Q}^-(o_j)}$, where $\tilde{p}(o_j) =p(o_j)\triangle_o$.
		\STATE Construct   $\phi_{\cM_Q}^-(\alpha)$: $\phi_{\cM_Q}^-(\alpha) \gets \sum_{o_j} e^{\log \tilde{q}(o_j) + i\alpha L_{Q, P}^-(o_j)}$, where $\tilde{q}(o_j) =q(o_j)\triangle_o$.
		\STATE Return $\phi_{\cM_P}^-$ and $\phi_{\cM_Q}^-$.
	\end{algorithmic}
\end{algorithm}

The main challenge in approximating $\phi$-function for continuous mechanisms is that an upper/lower bound of $\phi$-function does not necessarily attain the upper/lower bound of privacy costs.
Motivated by recent work~\citep{koskela2020tight} that truncates the privacy loss R.V. range to discretize subsample Gaussian mechanism, we consider discretizing the output domain. Our choice on the output domain  instead of the privacy loss R.V. range is because many recent advances in communication-efficient private learning require the output space to be discrete.
Note that the output domain  is the output domain of the dominating pair. There always exists a one-dimension tightly dominating pair for any mechanisms as we have shown in the main results.
The approximation procedure given in Algorithm~\ref{alg: appro} takes as input a truncated output domain $[-S, S]\subset \cO$, the partition parameter $N$ and the pdf measure of two privacy loss random variables. The algorithm first introduces the grid approximation of privacy loss RVs using the following definition.
\begin{definition}[Grid approximation of privacy loss R.V.]\label{def: appro}
	Given $N$ equidistant points $o_1, ..., o_N$ over the interval $[-S, S]$, We define $ L_{P,Q}^{-}(o_j) = \min_{t \in [j-1, j]} L_{P,Q}(o_t) \text{   and   }  L_{P,Q}^{+}(o_j) = \max_{t \in [j-1, j]} L_{P,Q}(o_t)$ as the grid approximation of the original privacy loss R.V., where  $L_{P, Q}(o_t)$ denotes the density of $ L_{P, Q}$ when evaluated at $o_t$.
\end{definition}
$L_{Q, P}^{-}(o_j)$ and $L_{Q, P}^{+}(o_j)$ are defined analogously. Next, the algorithm constructs  $\phi$-function approximations   $\phi_{\cM_P}^-$ and $\phi_{\cM_Q}^-$ using similar ideas as that in Definition~\ref{def: phi_discrete}, except that we replace privacy loss R.V. with their approximation alternatives. 
The idea behinds the approximation is to construct Riemann sum style lower and upper bounds of $\delta_\cM(\epsilon)$ using sampled points in the output interval.  We formalize the idea using the following lemma.


\begin{lemma}\label{lem: bound_delta}
	
	Consider the truncation parameter goes to infinity, i.e., $S \to \infty,$ and  privacy loss random variable $L_{P,Q}$ and $L_{Q,P}$ is a monotonical function of the output random variable $o$. We have for all $\epsilon > 0$,  the  privacy profile $\delta(\epsilon)$ of a continuous mechanism $\cM$ is bounded by
	\[\delta_{\cM_{min}}(\epsilon) \leq \delta_\cM(\epsilon) \leq \delta_{\cM_{max}}(\epsilon)\]
	where $\delta_{\cM_{min}}(\epsilon)$ is constructed using Algorithm~\ref{alg: ours} with$(\phi^-_{\cM_P}, \phi^-_{\cM_Q})$ pair and $\delta_{\cM_{max}}(\epsilon)$ is constructed using $(\phi^+_{\cM_P}, \phi^+_{\cM_Q})$ pair.
\end{lemma}
\begin{proof}
	
The proof sketch is first to show that the CDF $F_{L_{P, Q}} (\epsilon) $ is always smaller bounded by $ F_{L_{P, Q}^-} (\epsilon)$ when the monotonical condition is satisfied.  Then we rewrite the privacy profile into the CDF forms and demonstrate that a larger CDF leads to a smaller $\delta$ for any $\epsilon>0$. 
	\begin{align}
		F_{L_{P,Q}}(\epsilon) &= \int_{-\infty}^{\infty} \mathbb{I}\bigg[\log\bigg(\frac{p(o)}{q(o)}\bigg) \leq \epsilon \bigg]p(o)do\\
		&\leq \sum_{j \in \mathbb{Z}} \triangle_o \mathbb{I}\bigg[\log\bigg(\min_{t \in [j-1, j]}\frac{p(o_t)}{q(o_t)}\bigg)\leq \epsilon\bigg]p(o_j) \\
		&=F_{L^-_{P,Q}}(\epsilon)
	\end{align}

	Note that the indicator function preserves the monotonic property, which implies that we can lower bound $F_{L_{P,Q}}(\epsilon )$ using the left Riemann sum.
	Analogously, the CDF $F_{L_{Q,P}}(-\epsilon)$ is smaller than $F_{L^-_{Q,P}}(-\epsilon)$ for all $\epsilon>0$. Therefore, we have 
	\begin{align*}
		\delta_{\cM}(\epsilon) &= \sup_{D \simeq D'} \bigg( \pr[L_{P,Q}> \epsilon] -e^{\epsilon}\pr[L_{Q,P} < -\epsilon]\bigg)\\
		&=1- \sup_{P \simeq Q} \bigg( F_{L_{P,Q}}(\epsilon) +e^{\epsilon}F_{L_{Q, P}}(-\epsilon)\bigg)\\
		& \geq 1- \sup_{P\simeq Q} \bigg( F_{L^-_{P,Q}}(\epsilon) +e^{\epsilon}F_{L^-_{Q,P}}(-\epsilon)\bigg)\\
		&=\delta_{\cM_{\min}}(\epsilon)
	\end{align*}
\end{proof}
The following corollary allows us to upper and lower bound privacy cost over composition.
\begin{corollary}
	Consider a simple composition of $\cM_1$ and $\cM_2$ on datasets $D$ and $D'$, we have 
	$$
	\delta_{\cM_{1_{\min}}+\cM_{2_{\min}}}(\epsilon) \leq \delta_{\cM_1 + \cM_2}(\epsilon) \leq \delta_{\cM_{1_{\max}}+\cM_{2_{\max}}}(\epsilon)
	$$
\end{corollary}
\begin{proof}
	We overload the $L_{P,Q}$ to denote the privacy R.V. over composition. The lower bound of the CDF $F_{L_{P, Q}}$ is given as follows:
	\begin{align*}
		&	F_{L_{P,Q}}(\epsilon)  = \int_{-\infty}^{\infty} \mathbb{I}\bigg[ \log\bigg( \frac{p_{\cM_1 \cdot \cM_2}(o)}{q_{\cM_1 \cdot \cM_2}(o)}\bigg) \leq \epsilon \bigg] p_{\cM_1, \cM_2}(o)do\\
		&=\triangle_o\sum_{j \in \mathbb{Z}} \sum_{r \in \mathbb{Z}} \mathbb{I}\bigg[\log\bigg(  \frac{p_{\cM_1}(o_r)}{q_{\cM_1}(o_r)}\bigg) + \log\bigg(\frac{p_{\cM_2}(o_j - o_r)}{q_{\cM_2}(o_j -o_r)}\bigg) \\
		&	\leq \epsilon \bigg]p_{\cM_{1} \cM_2} (o_j)\\
		&\leq \triangle_o\sum_{j \in \mathbb{Z}} \sum_{r \in \mathbb{Z}} \mathbb{I}\bigg [\log\bigg(
		\min_{t_1 \in [r-1, r]} \frac{p_{\cM_1}(o_{t_1})}{q_{\cM_1}(o_{t_1})}\bigg)  \\
		&+ \log\bigg( \min_{t_2 \in [o_j - o_r -1, o_j -o_r]}\frac{p_{\cM_2}(t_2)}{q_{\cM_2}(t_2)}\bigg) 
		\leq \epsilon \bigg]p_{\cM_{1} \cM_2} (o_j)\\
		&=	F_{L^-_{P, Q}}(\epsilon)
	\end{align*}
\end{proof}

\noindent\textbf{Gaussian quadrature:} 
We apply Gaussian quadrature to efficiently evaluate integral in computing CDFs. Note that the integral is defined over an infinite integral (see Theorem~\ref{thm: levy}), we will use the following lemma to convert the integral range to $[-1,1]$ before we apply Gaussian quadrature.

\begin{lemma}[Integrals over infinite intervals]
	Let $f(x)$ is defined over the infinite interval. We have
	$$
	\int_{-\infty}^\infty f(x)dx = \int_{-1}^{+1}f\bigg(\frac{t}{1-t^2}\bigg)\frac{1+t^2}{(1-t^2)^2}dt
	$$
\end{lemma}
By a change of variables, we can convert the infinite integral to a finite integral, which can be easily implemented using numerical integration methods.

\noindent\textbf{Double quadrature:} To improve the time complexity of Algorithm~\ref{alg: appro}, which is linear with a sufficiently large $N$, we next apply Gaussian quadrature to approximate $\phi$-function when its closed-form expression is not available. We call this algorithm ``Double quadrature'', as we use it twice, one is in the approximation for $\phi$-function, and another is for the CDF computation. In the experiment section, we show that the ``Double quadrature'' algorithm matches the  Fourier Accountant approach~\citep{koskela2021computing} while only samples hundreds of points in evaluating Poisson Subsample Mechanisms.

\subsection{Error analysis}
In this section, we provide the end-to-end error analysis of AFA by looking into the following three scenarios. 

\begin{enumerate}
	\item $\phi$-functions have closed-form expressions.
	\item $\phi$-functions with discretization is used (see Algorithm~\ref{alg: appro}).
	\item Double quadrature algorithm: $\phi$-function is approximated  using Gaussian quarature.
\end{enumerate}
\noindent\textbf{Error analysis  closed-form $\phi$ functions.}
When we have closed-form $\phi$-functions (see Exp1 and Exp2), the numerical error is only caused by the quadrature method, which is used to convert $\phi$-function to CDFs. Therefore, we can tap into the classical results from numerical analysis that bounds the error in quadrature methods (see, e.g., Chapter 7 of Conte and de Boor ``Elementary Numerical Analysis'') with an asymptotic scaling more or less $O(1/n^\alpha)$ for integrating an $\alpha$th times differentiable functions, and faster for more advanced rules, see the following Lemma. 

\begin{lemma}\citep{stoer2002interpolation}
	Let the function $f(\cdot)$ has $2n$ continuous derivatives over the integral $[a, b]$ and $n$ is the number sample points. We have the error estimate
	\[\frac{(b-a)^{2n+1}(n!)^4}{(2n+1)[(2n)!]^3} f^{(2n)}(\zeta),  a<\zeta<b \]
\end{lemma}
From a practical standpoint,  ``Scipy.integrate'' allows us to specify a desired error tolerance (e.g., $O(10^{-14})$ used in experiments). The error induced in CDFs will be amplified by $e^\epsilon$ in the final $\delta(\epsilon)$ evaluation according to Lemma~\ref{lem:equivalent_defs}, which is still negligible in practice. The inverse ($\epsilon$ as a function of $\delta$) requires an additional binary search, which calls $\delta(\epsilon)$ a handful of times.

 In the cases when the $\phi$-function does not have a closed-form expression, we proposed two approaches to approximate the -function: Algorithm 2 and Double Quadrature. In Algorithm 2, we discretize the support of the dominating pairs  using an equispaced grid approximation. 

\noindent\textbf{Error analysis with approximated $\phi$ functions.} We analyzed the error caused by truncation, approximation of Algorithm 2 as follows.  
\begin{enumerate}
	\item The error caused by truncation (ignore $o\geq S$ or $o\leq -S$).
	\item The error arising from Riemann Sum approximation.
	\item The error caused by using Gaussian quadrature to compute CDF.
\end{enumerate}
The first two errors arising from the approximation of $\phi$-function and the third error is the numerical error when we apply Gaussian quadrature to compute the integral in privacy profile. 
The third term is often negligible as we discuss earlier. 
When $p$ and $q$ are determined, we bound the tail integral $\int_{S}^\infty \max(p(o), q(o))do$ and  $\int_{-S}^{-\infty} \max(p(o), q(o))do$ using the Chernoff bound and denote it is upper bounded by $\delta_{tail}$. Then a union bound over $k$ compositions will bound all bad events that the output happens to be out of $[-S, S]^k$.  Lastly, we estimate the total error by subtracting and adding $k\delta_{tail}$ to the lower and the upper  Riemann sums bound, respectively.
\begin{theorem}\label{thm: error_analysis}
	Denote the privacy profile of machanism $\cM$ over $k$-fold composition as $\delta_\cM^{(k)}$.
	Let $\delta_{\cM_{min}}^{(k)}(\epsilon)$ and $\delta_{\cM_{max}}^{(k)}(\epsilon)$ denote the lower and upper bounds of Riemann sum approximation over $k$-fold composition. Denote $\delta_{tail}$ be the upper bound of tail integral, i.e.,$\pr[o>S] + \pr[o<S] \leq \delta_{tail}$. Then we have
	\[\delta_{\cM_{min}}^{(k)}(\epsilon) - k\delta_{tail} < \delta_{\cM}^{(k)}< k\delta_{tail} + \delta_{\cM_{max}}^{(k)}(\epsilon)\]
\end{theorem}
\noindent\textbf{Double Quadrature}
For the second approach --- adaptive approximation via double quadrature, we do not have an asymptotic analysis of the bounds but `scipy.integrate.dblquad’ does provide us valid bounds. The “double quadrature” approach is the best-performing algorithm we recommend in practice.

\subsection{Experimental details in Exp 3}

Theorem~\ref{thm:amplification_by_sampling} allows us to consider the following one dimension\footnote{the error analysis of the $(P, (1-\gamma)P + \gamma Q)$ dominating pair is similar} distribution as the worst-case pair neighboring distribution for the Poisson subsampling Gaussian mechanism.
\begin{align*}
	p(o) &= \gamma \frac{1}{\sqrt{2\pi \sigma^2}}e^{\frac{(o-1)^2}{-2\sigma^2}} + (1-\gamma)\frac{1}{\sqrt{2\pi \sigma^2}}e^{\frac{o^2}{-2\sigma^2}}\\
	q(o)&= \frac{1}{\sqrt{2\pi \sigma^2}}e^{\frac{o^2}{-2\sigma^2}}
\end{align*}
Therefore, the pairing privacy loss RV is given by
$$
L_{P,Q} = \log \bigg(\gamma e^{\frac{2o-1}{2\sigma^2}}+1-\gamma\bigg),  L_{Q, P}= -\log \bigg(\gamma e^{\frac{2o-1}{2\sigma^2}}+1-\gamma \bigg) 
$$

We cannot derive a closed-form expression of $\phi$-function naively for the above privacy loss RVs. Therefore, we consider the discretization method (Algorithm~\ref{alg: appro}) to approximate the Poisson subsampling.  

We set $S = 100$ and $N=10^5$ to discretize the exact integral of $o$, which is $[-\infty, \infty]$ for the Gaussian mechanism. We use $\gamma=0.01$ and $\sigma =2.0$, for number of compositions up to $1500$.
We now provide the error analysis of our Poisson subsample Gaussian experiments. Recall the error analysis in Theorem~\ref{thm: error_analysis}, we first bound the error caused by truncation (ignore $o\geq S$ or $o\leq -S$).

The error caused by trunction consists of the tail integral $t_1 := \int_{S}^\infty \max(p(o), q(o))do$ and $t_2:= \int_{-S}^{-\infty} \max(p(o), q(o))do$.   We can upper bound $t_1$ and $ t_2$ using the tail bound of Gaussian distribution.
\[t_1 \leq \pr[\cN(1, \sigma^2)>S] \leq e^{\frac{(S-1)^2}{-2\sigma^2}}\]

The first inequality is because $p(o)$ is a mixture of two Gaussian distribution $\cN(0, \sigma^2)$ and $\cN(1 ,\sigma^2)$.  The second inequality follows  the tail bound of Gaussian distribution $\pr[\cN(0, \sigma)>x]\leq e^{\frac{-x^2}{2\sigma^2}}$.  Similarly, we can upper bound $t_2$ by $e^{\frac{S^2}{-2\sigma^2}}$.

We use $\delta_{tail}$ to denote the failure probability when $o$ happens to be out of the range $[-S, S]$. Here, we have $\delta_{tail} \leq e^{\frac{S^2}{-2\sigma^2}}+ e^{\frac{(S-1)^2}{-2\sigma^2}}$.
Substituting $S= 100, k= 1500$ and $\sigma=2.0$ into Theorem~\ref{thm: error_analysis},   we have $k\cdot \delta_{tail}$ is upper bounded by $e^{-1000}$, which is neglectable in the $\delta$ term.   
For the error caused by discretization, we plot the valid lower and upper bound using Algorithm~\ref{alg: appro}.

In ``Double quadrature'', we apply Gaussian quadrature to compute $\phi(\cM)$ and $\phi'(\cM)$ directly. That is,  we apply Gaussian quadrature to solve the integration
\[\phi(\cM) (\alpha)= \int_{-\infty}^\infty e^{i\alpha \log(p(o)/q(o))} q(o)do\]
and 
\[\phi'(\cM) (\alpha)= \int_{-\infty}^\infty e^{i\alpha \log(q(o)/p(o))} p(o)do\]

Though we did not include an error analysis of the Double quadrature algorithm, the algorithm exactly matches the result from \citet{koskela2020tight} and its computation time for each  $\delta(\epsilon)$ query is only around $0.2$ sec.

\section{An ``optimal'' Renyi DP to DP conversion?}\label{sec:rdp2dp}
\label{sec:rdp2dp}
In this section, we provide the detailed description of how we generated the improved ``optimal'' conversion rule in Figure~\ref{fig: lossy_rdp} based on an extension of the technique of \citet{balle2020hypothesis}.  

Specifically, \citet{balle2020hypothesis} shows that $(\alpha,\epsilon)$-RDP implies $f$-DP for any tradeoff function $f$ that lower bounds the following tradeoff region:
$$
\left\{(x,y) \in [0,1]^2 \bigg|\begin{array}{lr}
x^\alpha (1-y)^{1-\alpha} + (1-x)^{\alpha}y^{1-\alpha}\leq e^{(\alpha-1)\epsilon(\alpha)}\\
y^\alpha (1-x)^{1-\alpha} + (1-y)^\alpha x^{1-\alpha} \leq e^{(\alpha-1)\epsilon(\alpha)}
\end{array}\right\}.
$$

Then one can further convert this $f$-DP to  $(\epsilon,\delta)$-DP according to our formula in Section~\ref{sec:conversion_results} (originally due to \citep{dong2019gaussian}.)  


The main improvement that we propose is to consider the mechanism specific version of the same conversion rule, which converts the RDP function $\epsilon_\cM(\cdot)$ satisfied by a mechanism $\cM$ to an $f$-DP of $\cM$, which involves taking the pointwise maximum of all $f$ functions implied by each $(\alpha,\epsilon_\cM(\alpha))$-RDP.   The key to obtain the conversion that we have shown in Figure~\ref{fig: lossy_rdp} was to consider and extended version of RDP that also includes $0 <\alpha <1$, which we find to have a nontrivial effect in the resulting tradeoff function.


In the remainder of the section, we will first explain how the two-stage conversion works in Section~\ref{sec:rdp2fdp} and Section~\ref{sec:fDP2DP} and then comment on whether the rule can be improved in Section~\ref{sec:weak_optimality}.

	\subsection{RDP to $f$-DP}\label{sec:rdp2fdp}
	The hypothesis testing interpretation ($f$-DP) of RDP is not entirely tight. Here we present a simplified derivation of the tradeoff function $f$ implied by RDP via closedness to post-processing.
	
Consider any hypothesis testing procedure $h: \cO \rightarrow \{0,1\}$ that uses the output $o$ of an $(\alpha,\epsilon)$-RDP mechanism. ``$1$'' denotes ``Rejecting the null hypothesis'' that individual $z$ is not in the dataset, indicating that $h$ predicts that $z$ is in the dataset.  ``$0$'' denotes the complement event of ``Failing to reject the null hypothesis'', indicating that $h$ predicts that $z$ is not in the dataset.  By the closure to post-processing property, $h(o)$ satisfies $(\alpha,\epsilon)$-RDP. By definition of RDP,
\begin{equation}\label{eq:rdp_f_function_master}
\frac{1}{\alpha-1} \log \E_{h(o)\sim q}\left[(\frac{p}{q})^\alpha \right]\leq \epsilon
\end{equation}
for all pairs of neighboring datasets that induce distributions $p,q$.

Let $q$ be the distribution where individual $z$ is not in the dataset and $p$ otherwise. 
Let $x$ denote the probability of false positive (Type I error) --- $z$ is not in the dataset but the prediction is $1$; and $y$ denote the probability of false negative (Type II error) --- $z$ is in the dataset but the prediction is $0$. 

For $\alpha > 1$, \eqref{eq:rdp_f_function_master} is equivalent to 
$$
\begin{aligned}
(1-y)^{\alpha} x^{1-\alpha} + y^{\alpha} (1-x)^{1-\alpha} \leq e^{(\alpha-1)\epsilon},\\
x^{\alpha} (1-y)^{1-\alpha} + (1-x)^{\alpha} y^{1-\alpha} \leq e^{(\alpha-1)\epsilon},
\end{aligned}
$$
where the first constraint follows by taking the moments of the density ratio of the binary random variable $h(o)$, by noting that the event for prediction $1$ is false positive under $q$ but true positive under $q$. The second constraint follows from swapping $p,q$.

When $\alpha = 1$, by the definition of KL-divergence,  \eqref{eq:rdp_f_function_master} is equivalent to 
$$
\begin{aligned}
x \log(\frac{x}{1-y}) + (1-x)\log(\frac{1-x}{y}) \leq \epsilon,\\
y \log(\frac{y}{1-x}) + (1-y)\log(\frac{1-y}{x}) \leq \epsilon.
\end{aligned}
$$

Finally, when $0 <\alpha < 1$, \eqref{eq:rdp_f_function_master} is equivalent to 
$$
\begin{aligned}
(1-y)^{\alpha} x^{1-\alpha} + y^{\alpha} (1-x)^{1-\alpha} \geq  e^{(\alpha-1)\epsilon}\\
x^{\alpha} (1-y)^{1-\alpha} + (1-x)^{\alpha} y^{1-\alpha} \geq e^{(\alpha-1)\epsilon}.
\end{aligned}
$$
Note that the only difference from the case when $\alpha >1$ is the direction of the inequality. Also note that the symmetry of the two inequalities ensures that it suffices to consider $\alpha\geq 0.5$\footnote{Let $0<\alpha<0.5$ and $\alpha' = 1-\alpha$. Notice that $(1-y)^{\alpha} x^{1-\alpha} + y^\alpha (1-x)^{1-\alpha} = x^{\alpha'} (1-y)^{1-\alpha'} + (1-x)^{\alpha'} y^{1-\alpha}$.  Next, check that for all positive $\epsilon$ and $0<\alpha<0.5$, $e^{(\alpha-1)\epsilon}< e^{(\alpha'-1)\epsilon}$. In other words, the bound with $\alpha$ in $(0,0.5)$ is never active. }.
	
The f-DP of a mechanism satisfying $\epsilon(\alpha)$-RDP for a family $\alpha$ is therefore the pointwise maximum of the resulting $f$ function for all $\alpha$.
	
	\subsection{$f$-DP to $(\epsilon,\delta)$-DP}\label{sec:fDP2DP}
	
fDP is related to $(\epsilon,\delta)$-DP in the following lemma.
\begin{lemma}
	Let $f$ be the lower bound of Type II error given Type I error, then mechanisms that satisfy fDP with function $f$ also obeys a family of $(\epsilon(x),\delta(x))$-DP for all $x\in [0,1]$ such that 
	\begin{align*}
	\delta(x) &= 1- f(x) - (-\partial f(x)) x\\
	\epsilon(x) &= \log (-\partial f(x)) 
	\end{align*}
	where $\partial f(x)$ is any subgradient of $f$ at $x$. Recall that by definition of subgradient, 
	$g$ is a subgradient of $f$ at $x$ if for all $y \in \mathrm{Dom}(f)$, $f(y) \geq f(x) + g\cdot (y-x)$ (note that $\mathrm{Dom}(f) = [0,1]$ for trade-off function $f$).
	\end{lemma}

Numerically stable computation of $\epsilon$ given $\delta$ or $\delta$ given $\epsilon$ involves working with $\partial f$ and $1-f$ in logarithmic scale.
Specifically, we find $x$ such that there exists subgradient $g\in \partial f(x)$ such that 
\begin{equation}\label{eq:delta_equation}
\log(\delta)  \geq  \log( e^{\log(1-f(x))} -  e^{\log(-g) + \log(x)}).
\end{equation}
Then it is true that  $\epsilon(\delta)  =  \log(-g)$.
In fact, a more general procedure finds an $x$ (why any feasible $x$ works is left as an exercise) such that $\partial f(x)$ contains $g$ satisfying \eqref{eq:delta_equation}.
It then follows that 
$$\epsilon(\delta)  = \min_{g\in\partial f(x)} \log(-g).$$
When $f(x)$ is differentiable everywhere, we can solve \eqref{eq:delta_equation} as a nonlinear equation and then we can write $\epsilon(\delta)  =  \log(- f'(x(\delta)))$ if $f'(x(\delta)) < -1 $, and $\epsilon(\delta) = 0$ otherwise.

Similarly, $\delta(\epsilon)$ can be found by solving the nonlinear equation  
$\log(-\partial f(x)) = \epsilon$ for $x$ and then plug into \eqref{eq:delta_equation}:
$$
\delta(\epsilon) = e^{\log(1-f(x(\epsilon)))} - e^{\log(-\partial f(x(\epsilon))) + \log(x(\epsilon))}.
$$
In other word, provided that $\log(1-f(x))$ and $\log(-\partial f(x))$ admit an analytical implementation, we can convert $f$-DP to $(\epsilon,\delta)$-DP in a numerically stable fashion.

\subsection{Optimality of this conversion rule?}\label{sec:weak_optimality}

A natural question to ask is that whether the aforementioned conversion rule from RDP to $(\epsilon,\delta)$-DP is optimal.  The answer is yes and no. Let us explain.

From Figure~\ref{fig: lossy_rdp}, we can clearly see that for the randomized response mechanism, the resulting conversion matches exactly with the exact $(\epsilon,\delta)$-DP obtained via the privacy-profile.  

Moreover, observe that the converted $f$-function of the Gaussian mechanism touches that of the randomized response mechanism which satisfies the same RDP bound for all $\alpha > 0$. For this reason, we know that for Gaussian mechanism, the conversion rule cannot be improved in any ways that strictly improves the stated conversion rule at all inputs (Type I error).

This example, however, does not rule out the possibility of improving the RDP-implied $f$-DP elsewhere for Gaussian mechanism. Therefore, it is unclear whether the proposed RDP-to-DP conversion rule is optimal in the strong sense:
\begin{center}
	\textsf{Is it optimal for all RDP functions and all input (type I error) at the same time?}
\end{center}
Our conjecture is positive, but it is beyond the scope of the current paper to formally prove this. 

A promising direction is to compare our approach to the optimal conversion rule proposed by \citet{asoodeh2021three} (which uses a very different approach to derive almost the same formula as that in \citep{balle2020hypothesis}).


\section{Omitted proofs in \Cref{sec:conversion_results}} 
\label{sec:omitted_proofs_in_sec:conversion_results}

\begin{lemma}\label{lem:HSpq}
	$$H_\alpha(Q\|P) = \alpha H_{\alpha^{-1}}(P\|Q)-\alpha+1.$$
\end{lemma}
\begin{proof}[Proof of \Cref{lem:HSpq}]
	Let $p,q$ be the Radon--Nikodym derivative of $P$ and $Q$ with respect to a common dominating measure $\mu$ (say $\frac{1}{2}P+\frac{1}{2}Q$). Then $H_\alpha(P\|Q) = \int (p-\alpha q)_+\diff \mu$ and $H_\alpha(Q\|P) = \int (q-\alpha p)_+\diff \mu$. We will drop the $\diff \mu$ notation in the integrals for convenience.
	\begin{align*}
		H_\alpha(Q\|P) &= \int (q-\alpha p)_+
		=\alpha\int (\alpha^{-1}q- p)_+\\
		&=\alpha\int (\alpha^{-1}q- p)+\alpha\int (p-\alpha^{-1}q)_+\\
		&=\alpha(\alpha^{-1}-1)+\alpha H_{\alpha^{-1}}(P\|Q)
		= \alpha H_{\alpha^{-1}}(P\|Q)-\alpha+1
	\end{align*}
	where we used the fact that $a_+-(-a)_+=a$ on the second line.
\end{proof}

\begin{proof}[Proof of \Cref{lem:HSfromFG}]
	Let $A=\{\omega:\frac{\diff P}{\diff Q}(\omega) \geqslant\alpha\} = \{\omega:\log\frac{\diff Q}{\diff P}(\omega) \leqslant \log\alpha^{-1}\}$. We have
	\begin{align*}
		H_\alpha(P\|Q)
		&= \E_{\omega\sim Q}[ (\frac{\diff P}{\diff Q}(\omega) - \alpha)_+ ]\\
		&= \E_{\omega\sim Q}[ (\frac{\diff P}{\diff Q}(\omega) - \alpha)\cdot 1_{A} ]\\
		&= P[A]-\alpha Q[A]= F(-\log\alpha)-\alpha G(-\log \alpha).
	\end{align*}
	The second identity can be obtained in a similar fashion.
\end{proof}

\begin{proof}[Proof of \Cref{lem:ffromFG}]
	Follow directly from \Cref{lem:tool}.
\end{proof}

\begin{proof}[Proof of \Cref{lem:FGfromHS}]
	The first expression is the direct consequence of $f'\big(1-F(x)\big)=-\e^x$ and the well-known fact in convex analysis that for a convex function $f$, $f'$ and $(f^*)'$ are inverse functions of each other.

	Let $t=-\e^x$. Then from \Cref{lem:tool} we have $H_{\e^x}(Q\|P) = 1+f^*(t)$ and hence $\frac{\diff}{\diff t}H_{\e^x}(Q\|P) = (f^*)'(t)$. Therefore,
	\begin{align*}
		F(x)&=1-(f^*)'(t) = 1-\frac{\diff}{\diff t}H_{\e^x}(Q\|P)= 1-\frac{\diff x}{\diff t}\cdot\frac{\diff}{\diff x}H_{\e^x}(Q\|P) \\
		&= 1-\left(\frac{\diff t}{\diff x}\right)^{-1}\cdot\frac{\diff}{\diff x}H_{\e^x}(Q\|P) = 1+\e^{-x}\cdot\frac{\diff}{\diff x}H_{\e^x}(Q\|P)
	\end{align*}

	Now the first identity about $G$ is a direct consequence of \Cref{lem:ffromFG}, and the second one is a direct consequence of \Cref{lem:HSfromFG}.
\end{proof}

\begin{lemma}\label{lem:GfromF}
	$G(x)=\int_{-\infty}^x \e^{t}\diff F(t)$.
\end{lemma}
\begin{proof}[Proof of \Cref{lem:GfromF}]
	Let $h:\R\to\R$ be a Borel measurable function. Then
	\begin{align*}
		\int h(x)\diff G(x)
		=\int h\left(\log\tfrac{\diff Q}{\diff P}(\omega)\right)\diff Q(\omega)
		=\int h\left(\log\tfrac{\diff Q}{\diff P}(\omega)\right)\cdot \tfrac{\diff Q}{\diff P}(\omega)\diff P(\omega) =\int h(x)\e^x\diff F(x)
	\end{align*}
	The choice of $h$ as the indicator function of $(-\infty, x]$ yields the second identity.
\end{proof}

\begin{lemma}\label{lem:tool}
	The functions $f,F,G$ and the hockey-stick divergence have the following relations:
	\begin{equation*}
	\begin{gathered}
		f(1-F(x)) = G(x)\\
		f'\big(1-F(x)\big)=-\e^x\\
		H_{\e^\eps}(Q\|P) = 1+f^*(-\e^\eps) 
	\end{gathered}
	\end{equation*}
\end{lemma}
\begin{proof}[Proof of \Cref{lem:tool}]
	The first identity follows from the definition of the trade-off function and Neyman--Pearson lemma. In fact, $1-F(x)$ and $G(x)$ are the type I and type II errors of the likelihood ratio test with threshold at $x$.
	Taking derivative with respect to $x$ on both sides of the first identity, we have
	\[f'(1-F(x))\cdot (-F'(x))=G'(x).\]
	Now the second identity follows by plugging in \Cref{lem:GfromF}.
\end{proof}

\begin{proof}[Proof of \Cref{lem:f.ch.f}]
	It is known that hockey-stick divergence is determined by the trade-off function, namely $H_{\e^\eps}(Q\|P) = 1+f^*(-\e^\eps)$. By \Cref{lem:FGfromHS}, we know the distributions of $\log\frac{\diff P}{\diff Q}$ (under $P$ and $Q$ respectively) are also determined by the trade-off function $f$. Therefore, $\phi$ and $\phi'$ are determined by $f$. In particular, if we can find $P',Q'$ such that $f=T[P',Q']$, then we can use $\phi(t)=\E_{P'}\e^{it\log\frac{\diff P'}{\diff Q'}}$ and $\phi(t)=\E_{Q'}\e^{it\log\frac{\diff Q'}{\diff P'}}$ to compute $\phi$ and $\phi'$.

	From the proof of Proposition 2.2 of \cite{dong2019gaussian}, we know that we can pick $P'$ as the uniform distribution over $[0,1]$ and $Q'$ has density $-f'(1-x) = |f'(1-x)|$ on $[0,1]$. Therefore,
	\begin{align*}
		\phi(t)
		&=\E_{P'}\e^{it\log\frac{\diff P'}{\diff Q'}}=\int_0^1\e^{-it\log\frac{\diff Q'}{\diff P'}}\diff x\\
		&=\int_0^1\e^{-it\log|f'(1-x)|}\diff x=\int_0^1\e^{-it\log|f'(x)|}\diff x
	\end{align*}
	The second identity can be proved similarly.
\end{proof}

\end{document}